\documentclass[sigconf]{acmart}

\usepackage{amsthm}
\usepackage{booktabs}
\usepackage{multirow}
\usepackage{algorithm}
\usepackage{algpseudocode}
\usepackage{placeins}

\newtheorem{theorem}{Theorem}
\newtheorem{proposition}[theorem]{Proposition}
\newtheorem{lemma}[theorem]{Lemma}

\AtBeginDocument{%
  }

\copyrightyear{2026}
\acmYear{2026}
\setcopyright{othergov}
\acmConference[KDD '26]{Proceedings of the 32nd ACM SIGKDD Conference on Knowledge Discovery and Data Mining V.2}{August 09--13, 2026}{Jeju Island, Republic of Korea}
\acmBooktitle{Proceedings of the 32nd ACM SIGKDD Conference on Knowledge Discovery and Data Mining V.2 (KDD '26), August 09--13, 2026, Jeju Island, Republic of Korea}
\acmDOI{10.1145/3770855.3818106}
\acmISBN{979-8-4007-2259-2/2026/08}

\begin{document}

\title{Scalable Uncertainty Quantification for Extreme Weather Forecasting via Empirical Neural Tangent Kernels}

\author{Jose Marie Antonio Mi\~{n}oza}
\affiliation{%
  \department{Center for AI Research}
  \institution{Department of Education}  
  \city{Makati}
  \country{Philippines}}
\email{ecair.jminoza@deped.gov.ph}

\author{Rex Gregor Laylo}
\affiliation{%
  \department{Center for AI Research}
  \institution{Department of Education}  
  \city{Makati}
  \country{Philippines}}
\email{ecair.rlaylo@deped.gov.ph}

\author{Sebastian C. Iba\~{n}ez}
\affiliation{%
  \department{Center for AI Research}
  \institution{Department of Education}  
  \city{Makati}
  \country{Philippines}}
\email{ecair.sibanez@deped.gov.ph}

\renewcommand{\shortauthors}{Mi\~{n}oza et al.}

\begin{abstract}
Deep learning weather models now match numerical weather prediction accuracy while running orders of magnitude faster, but produce deterministic forecasts without uncertainty estimates, a critical gap for high-stakes decisions during extreme weather events.
This paper proposes Neural Tangent Kernel-based uncertainty quantification (NTK-UQ) using last-layer empirical features.
Theoretical analysis predicts that UQ quality is architecture-dependent through two mechanisms.
First, a variance collapse mechanism explains when UQ fails: when the eigenvalue truncation rank approaches the effective rank of the feature space, the GP correction term consumes nearly all prior variance, destroying discrimination between tropical cyclones and routine conditions; architectures with concentrated spectra (spectral operators) require aggressive truncation ($k \leq 10$), while attention-based models tolerate full-rank computation.
Second, decomposition performance depends on the non-Gaussian, heavy-tailed structure of extreme weather: Independent Component Analysis exploits higher-order statistics (kurtosis, negentropy) to isolate heavy-tailed extreme-event features, achieving higher discrimination than singular value decomposition, which captures only second-order variance.
A data-driven selection rule chooses ICA or SVD from the feature eigenspectrum concentration ratio, correctly prescribing the superior decomposition for all four evaluated architectures.
Compared to split conformal prediction (the natural post-hoc baseline), NTK-UQ achieves 31--37\% sharper prediction intervals at 90\% coverage, and uniquely produces \emph{adaptive} intervals that scale with extreme event severity, which conformal prediction cannot achieve by construction.
The framework requires no retraining; inference-time uncertainty requires only a single matrix-vector product per sample.
\end{abstract}



\begin{teaserfigure}
  \includegraphics[width=\textwidth]{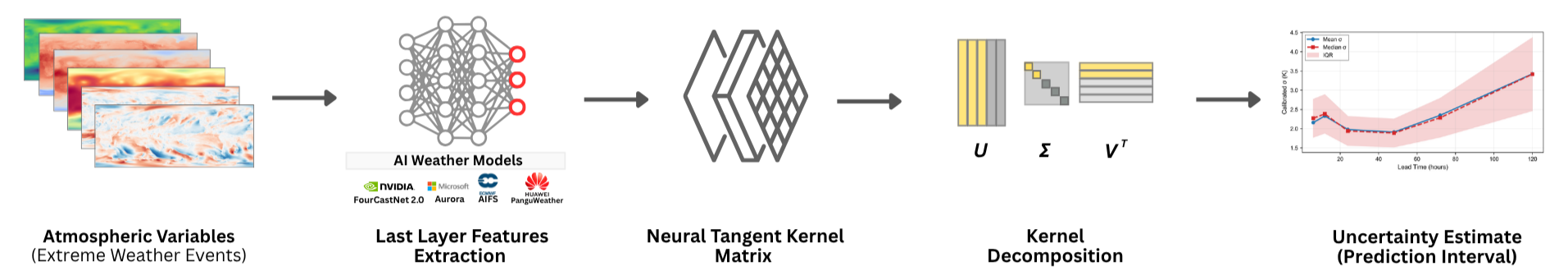}
  \caption{Overview of the NTK-UQ pipeline for extreme weather forecasting. Atmospheric variables from extreme weather events are processed through four foundation AI weather models (FourCastNetV2, Aurora, AIFS, Pangu-Weather) to extract last-layer features. These features construct the empirical Neural Tangent Kernel matrix, which is decomposed via SVD or ICA (shown: $U\Sigma V^\top$ decomposition) to obtain rank-$k$ approximation. At inference, the GP posterior variance formula yields calibrated prediction intervals that quantify epistemic uncertainty per variable.}
  \Description{Diagram showing the NTK-UQ pipeline: feature extraction from four AI weather models, empirical NTK kernel construction, SVD/ICA decomposition, and GP posterior uncertainty estimation at inference.}
  \label{fig:pipeline}
\end{teaserfigure}

\maketitle

\newcommand\kddavailabilityurl{https://doi.org/10.5281/zenodo.20499051}
\ifdefempty{\kddavailabilityurl}{}{
\begingroup\small\noindent\raggedright\textbf{Resource Availability:}\\
To foster reproducibility, the code, calibration matrices, and the EM-DAT initialization-date list used in this paper are publicly available at \url{\kddavailabilityurl}.
\endgroup
}

\section{Introduction}

Extreme weather events cause an estimated US\$143 billion per year in climate-attributable damages~\cite{newman2023global}, with the EM-DAT database recording 399 disasters in 2023 alone, affecting 93.1 million people~\cite{emdat2024}. Accurate forecasting of these events is essential, yet the value of a forecast depends not only on its accuracy but on knowing \emph{how much to trust it}, a question that requires calibrated uncertainty estimates.

Deep learning has transformed weather forecasting. Models such as FourCastNetV2~\cite{pathak2022fourcastnet}, Pangu-Weather~\cite{bi2023pangu}, GraphCast~\cite{lam2023graphcast}, and Aurora~\cite{bodnar2025aurora} now match or exceed the accuracy of traditional numerical weather prediction (NWP) systems while running orders of magnitude faster, generating 10-day global forecasts in seconds rather than hours. However, these models produce deterministic point forecasts without calibrated uncertainty estimates. Uncertainty quantification is essential for risk-sensitive applications: decision-makers require not only point predictions but probabilistic intervals that correlate with actual forecast errors. For extreme events where forecast errors carry the highest consequences, the absence of reliable uncertainty estimates limits model utility.

Existing approaches to uncertainty quantification (UQ) for neural networks face significant limitations when applied to large-scale weather models. Deep ensembles~\cite{lakshminarayanan2017simple} require training multiple copies of billion-parameter models from scratch, which is computationally prohibitive for foundation weather models. Monte Carlo dropout~\cite{gal2016dropout} can produce miscalibrated uncertainties~\cite{ovadia2019can} and requires architectural modifications incompatible with pre-trained checkpoints. Bayesian neural networks~\cite{blundell2015weight} add substantial memory and compute overhead, scaling poorly to operational-size models. Conformal prediction~\cite{angelopoulos2021gentle} provides distribution-free coverage guarantees but in its standard form produces uniform interval widths that do not correlate with actual forecast errors.

This paper proposes last-layer Neural Tangent Kernel (NTK) based uncertainty quantification for AI weather models. The key insight is that a weather model's last-layer features $\phi(x)$---learned from decades of ERA5 reanalysis---encode physically meaningful atmospheric structure. Under the last-layer NTK--GP correspondence, the feature kernel $K(x, x') = \phi(x)^\top \phi(x')$ acts as an \emph{ERA5-informed similarity measure}: a test input receives high uncertainty when its atmospheric state is unusual relative to both the model's learned feature manifold and the calibration distribution. This two-level epistemic signal is inaccessible to purely statistical baselines such as conformal prediction. Critically, UQ quality is architecture-dependent and decomposition-dependent: a data-driven selection rule determines whether Independent Component Analysis or Singular Value Decomposition is appropriate from the feature eigenspectrum, correctly prescribing the superior method without exhaustive comparison.

Throughout this paper, the term \emph{NTK uncertainty} refers to the posterior variance obtained by treating the frozen model's last-layer features as an \emph{empirical} Neural Tangent Kernel and applying Gaussian Process posterior theory. This usage differs from the full infinite-width NTK formulation and should be interpreted as a finite-width, post-hoc kernel approximation induced by the learned feature representations. The theoretical results in this paper (Propositions~\ref{prop:collapse}~and~\ref{prop:ica_theory}, Theorem~\ref{thm:hoeffding}) are proved directly under this empirical kernel without invoking the infinite-width limit; prior work~\cite{he2020bayesian,huang2023efficient} shows that finite-width networks behave approximately as kernel machines, and post-hoc calibration corrects for residual approximation error.

NTK-UQ has several properties that make it suitable for studying UQ across large-scale weather models. First, the method requires no model retraining or architectural changes; it works with any pre-trained checkpoint as a purely post-hoc procedure. Second, after one-time offline calibration, inference-time UQ requires only a matrix-vector product, adding minimal overhead to the forward pass. Third, uncertainties are computed per output variable, enabling variable-level uncertainty estimates.

Theoretical analysis predicts that UQ quality depends on both neural architecture (through eigenspectrum concentration) and decomposition method (through higher-order statistics exploitation). The framework is evaluated on four architecturally diverse AI weather models: FourCastNetV2 (SFNO), Pangu-Weather (Swin Transformer), Aurora (Perceiver), and AIFS (GNN-Transformer), using ERA5 reanalysis~\cite{hersbach2020era5} as ground truth. Evaluation focuses on extreme weather events from the EM-DAT International Disaster Database, including tropical cyclones, floods, droughts, and extreme temperature events. Experiments span forecast lead times from 6 to 120 hours. Results validate these predictions: uncertainty discrimination quality follows architecture-dependent patterns, with Independent Component Analysis achieving adaptive intervals that scale with extreme event severity, while singular value decomposition produces more uniform intervals that fail to distinguish tropical cyclone forecasts from routine conditions.

\paragraph{Contributions.} This paper makes five contributions: (1) \textbf{Variance Collapse Characterization}: formal analysis of how eigenvalue spectrum concentration causes UQ failure, with diagnostic criterion $R_k = C_k/P < 0.9$ for maintaining discrimination (Proposition~\ref{prop:collapse}), linking neural architecture (SFNO vs Transformer) to effective rank and optimal truncation strategy; (2) \textbf{Non-Gaussian Discrimination Theory}: explanation of why Independent Component Analysis outperforms singular value decomposition for extreme weather through higher-order statistics exploitation (Proposition~\ref{prop:ica_theory}), providing theoretical justification for decomposition method selection based on feature distribution properties; (3) \textbf{Architecture-UQ Interaction Framework}: systematic characterization of how neural architecture determines NTK eigenspectrum properties, which govern UQ quality, enabling predictive diagnosis without exhaustive experimentation; (4) \textbf{Decomposition Selection Rule}: Algorithm~\ref{alg:selection} provides a data-driven recipe that selects ICA or SVD from the feature eigenspectrum concentration ratio, correctly prescribing the superior method for all four evaluated architectures, validated against split conformal prediction with 31--37\% sharper intervals in 81\% of valid comparisons; and (5) \textbf{Empirical Validation}: evaluation across four foundation weather models (FourCastNetV2, Pangu-Weather, Aurora, AIFS) on 100 extreme weather events from EM-DAT confirms theoretical predictions and demonstrates that NTK-UQ produces adaptive intervals (CV~$> 0$) that conformal prediction cannot achieve by construction.

\section{Related Work}

AI weather foundation models~\cite{pathak2022fourcastnet,bi2023pangu,lam2023graphcast,bodnar2025aurora} now match numerical weather prediction accuracy. A subset produce probabilistic forecasts natively, but each at a cost: ECMWF's operational ensemble (ENS) requires 51-member perturbation runs at deployment; GenCast~\cite{price2025gencast} trains a diffusion model from scratch; and SEEDS~\cite{li2024seeds} requires a pre-existing ensemble to emulate. All are tied to specific architectures. By contrast, the large majority of AI weather checkpoints---including FourCastNetV2, Pangu-Weather, Aurora, and AIFS---are deterministic and lack native uncertainty estimates. NTK-UQ targets this majority: it applies post-hoc to any pre-trained deterministic checkpoint without retraining, enabling \emph{checkpoint reusability} across the rapidly growing ecosystem of foundation weather models.

Existing post-hoc UQ methods face significant barriers for billion-parameter models: deep ensembles~\cite{lakshminarayanan2017simple} require training multiple copies (prohibitively expensive), Bayesian methods~\cite{blundell2015weight,gal2016dropout} need architectural modifications and yield poorly calibrated uncertainties~\cite{ovadia2019can}, and conformal prediction~\cite{angelopoulos2021gentle} provides coverage guarantees but lacks per-sample discrimination.

The Neural Tangent Kernel~\cite{jacot2018neural} shows that infinitely wide networks behave as Gaussian Processes, enabling closed-form uncertainty quantification. For tractability, the last-layer empirical NTK uses the feature kernel $K(x, x') = \phi(x)^\top \phi(x')$ rather than the full gradient-based NTK. This coincides with last-layer Laplace approximation~\cite{mackay1992practical,daxberger2021laplace} for linear output heads. Recent work~\cite{he2020bayesian,huang2023efficient} demonstrates that NTK-based GP posteriors capture epistemic uncertainty even in finite-width networks. Unlike $\Delta$-UQ~\cite{thiagarajan2022single}, which requires retraining with anchor perturbation, NTK-UQ operates entirely post-hoc on pre-trained models. Detailed comparisons are provided in Appendix~\ref{sec:related_extended}.

\section{Method}

NTK-UQ is a framework for post-hoc uncertainty quantification in pre-trained neural weather models. The method consists of three phases: (1) last-layer feature extraction, (2) offline GP posterior construction via kernel decomposition, and (3) post-hoc scaling to achieve target coverage.

\subsection{Problem Setup}

Let $f_\theta: \mathcal{X} \to \mathcal{Y}$ be a pre-trained weather model that maps atmospheric states $x \in \mathcal{X} \subset \mathbb{R}^{C \times H \times W}$ to predictions $y \in \mathcal{Y} \subset \mathbb{R}^{C' \times H \times W}$, where $C$ and $C'$ are input and output channels, and $H \times W$ is the spatial grid. Given a calibration dataset $\mathcal{D}_{\text{cal}} = \{(x_i, y_i^*)\}_{i=1}^{N}$ with ground truth $y_i^*$ (used to construct the GP posterior and determine post-hoc scaling), the goal is to estimate predictive uncertainty $\sigma^2(x)$ such that prediction intervals achieve a target coverage level (e.g., 90\% of ground truth values fall within the 90\% prediction interval).

\subsection{Gaussian Process Interpretation}

Under the last-layer NTK–GP correspondence, a neural network with last-layer features $\phi(x)$ can be viewed as a Gaussian Process:
\begin{equation}\label{eq:ntk_gp}
    f(x) \sim \mathcal{GP}(0, K(x, x')),
\end{equation}
where $K(x, x') = \phi(x)^\top \phi(x')$ is the last-layer empirical NTK (the feature kernel). Given calibration data, the GP predictive variance at a new point $x_*$ is:
\begin{equation}\label{eq:gp_variance}
    \sigma^2(x_*) = K(x_*, x_*) + \sigma_n^2 - \mathbf{k}_*^\top (K + \sigma_n^2 I)^{-1} \mathbf{k}_*,
\end{equation}
where $K(x_*, x_*) = \|\phi(x_*)\|^2$ is the prior variance, $\sigma_n^2$ is the observation noise variance, $\mathbf{k}_* = [K(x_*, x_1), \ldots, K(x_*, x_N)]^\top$ is the kernel vector to calibration points, and $K_{ij} = K(x_i, x_j)$ is the kernel matrix. The term $\sigma_n^2$ in the predictive variance accounts for irreducible noise in the observations and is estimated from the eigenvalue spectrum (Section~\ref{sec:svd}).

\paragraph{Interpretation.} The GP posterior variance has a natural interpretation: the kernel $K$ encodes a prior over the model's function space shaped by the calibration data geometry~\cite{lee2018deep}. For a test input $x_*$, the posterior variance $\sigma^2(x_*)$ quantifies similarity to the calibration distribution in feature space. When $x_*$ is dissimilar to calibration inputs, the correction term $\mathbf{k}_*^\top (K + \sigma_n^2 I)^{-1} \mathbf{k}_*$ is small, and the posterior variance remains close to the prior, yielding high epistemic uncertainty~\cite{kendall2017uncertainties}. Conversely, inputs similar to the calibration set receive large corrections, yielding low uncertainty.

Crucially, the feature map $\phi$ is not hand-crafted but learned from ERA5 reanalysis data spanning decades of global atmospheric observations. The kernel $K(x, x') = \phi(x)^\top \phi(x')$ is therefore an \emph{ERA5-informed similarity measure}: it encodes physically meaningful atmospheric structure---the geometry of realizable weather states as the model learned it from training. A test input receives high uncertainty when it is unusual relative to both (1) the ERA5-learned feature manifold (encoded in the frozen weights $\phi$) and (2) the calibration distribution (encoded in the GP posterior from $n$ samples). This two-level epistemic signal is inaccessible to purely statistical baselines such as conformal prediction, which operate in prediction-error space without access to the model's learned atmospheric representation.

\subsection{Last-Layer Feature Extraction}

Modern neural weather models decompose as $f_\theta = g_\psi \circ \phi_\omega$, where $\phi_\omega: \mathcal{X} \to \mathbb{R}^d$ extracts features and $g_\psi: \mathbb{R}^d \to \mathcal{Y}$ is the final prediction head. Last-layer features are extracted by registering forward hooks during inference. For spatial feature maps, multi-statistic aggregation computes six statistics per channel (mean, standard deviation, minimum, maximum, 25th and 75th percentiles), yielding a fixed-dimensional feature vector regardless of spatial resolution. Architecture-specific extraction details are provided in Appendix~\ref{sec:feature_extraction}.

\subsection{Kernel Decomposition}\label{sec:decomposition}

Direct inversion of the kernel matrix $K$ is $O(N^3)$, prohibitive for large calibration sets. Before decomposition, features are centered by subtracting the calibration mean: $\bar{\phi} = \frac{1}{N}\sum_i \phi(x_i)$ and $\tilde{\phi}(x) = \phi(x) - \bar{\phi}$. This removes the dominant mean direction from the spectrum, ensuring the decomposition captures directions of \emph{variation} rather than the shared mean signal. The choice of decomposition method significantly affects UQ quality; this work compares Singular Value Decomposition (SVD) and Independent Component Analysis (ICA) to characterize these effects.

\paragraph{SVD Decomposition.} The standard approach uses singular value decomposition on the centered feature matrix $\tilde{\Phi} \in \mathbb{R}^{N \times d}$:
\begin{equation}\label{eq:svd}
    \tilde{\Phi} = U S V^\top.
\end{equation}
This decomposition yields the centered kernel eigenstructure directly, since $\tilde{K} = \tilde{\Phi} \tilde{\Phi}^\top = U S^2 U^\top$, meaning the eigenvalues are $\lambda_j = s_j^2$ (squared singular values) and the eigenvectors are the columns of $U$. SVD finds orthogonal directions of \textit{maximum variance} in the feature space. Retaining only the top-$k$ components (where $k \ll N$) captures the dominant directions of variation.

\paragraph{ICA Decomposition.} An alternative approach uses Independent Component Analysis (ICA)~\cite{hyvarinen2000independent} to decompose features into \textit{statistically independent} components rather than orthogonal directions of maximum variance. ICA assumes that the observed features $\tilde{\phi}(x)$ are linear mixtures of independent source signals: $\tilde{\phi}(x) = A s(x)$ where $s(x)$ are the independent components and $A$ is the mixing matrix. The FastICA algorithm~\cite{hyvarinen1999fast} recovers the unmixing matrix $W = A^{-1}$ by maximizing non-Gaussianity of the sources, yielding components $s(x) = W\tilde{\phi}(x)$. Unlike SVD, which prioritizes variance, ICA exploits higher-order statistics (kurtosis, skewness) to separate sources.

For extreme weather events, ICA offers a critical advantage: while SVD's maximum-variance criterion biases the decomposition toward typical weather patterns (high-frequency, high-variance modes), ICA's independence criterion can isolate rare extreme event signatures that occur as independent factors in the joint distribution, even when they contribute low marginal variance. Empirical results show that ICA outperforms SVD for uncertainty quantification in extreme events for three of four architectures (AIFS, Aurora, FourCastNetV2); SVD achieves coverage for Pangu-Weather while ICA fails (Section~\ref{sec:results}).

The predictive variance formula becomes:
\begin{equation}\label{eq:predictive_variance}
    \sigma^2_{\text{raw}}(x_*) = \underbrace{\|\tilde{\phi}(x_*)\|^2 + \sigma_n^2}_{\text{prior + noise}} - \underbrace{\sum_{j=1}^{k} \frac{\lambda_j \cdot (\tilde{\phi}(x_*)^\top v_j)^2}{\lambda_j + \sigma_n^2}}_{\text{GP correction}},
\end{equation}
where $\tilde{\phi}(x_*) = \phi(x_*) - \bar{\phi}$ is the centered test feature, $v_j$ are the right singular vectors of $\tilde{\Phi}$, $\lambda_j = s_j^2$ are the eigenvalues, and $\sigma_n^2$ is the noise variance. Projections onto high-variance directions receive large corrections (low uncertainty); dissimilar inputs receive small corrections (high uncertainty).

\paragraph{Noise Variance Estimation.\!\!} The noise parameter $\sigma_n^2$ is estimated as the mean of the residual eigenvalues $\{\lambda_{k+1}, \ldots, \lambda_d\}$. When the top-$k$ components exhaust all variance this estimate approaches zero, causing the correction to fully cancel the prior and destroying discrimination. The method falls back to the mean of the retained eigenvalues as a regularization nugget~\cite{cressie1993statistics}, preserving posterior variation even when the feature space is low-rank.

\begin{proposition}[Variance Collapse]\label{prop:collapse}
Let $\tilde{\Phi} = U S V^\top$ be the SVD of the centered calibration features and define the correction-to-prior ratio $R_k = C_k/P$ where $C_k = \sum_{j=1}^k \lambda_j c_j^2/(\lambda_j + \sigma_n^2)$ and $P = \|\tilde{\phi}(x_*)\|^2$. When the noise regularizer $\sigma_n^2 > 0$, each shrinkage weight $w_j = \lambda_j/(\lambda_j + \sigma_n^2) < 1$, so $R_k < 1$ and $\sigma^2(x_*) > 0$ for all ranks $k$ -- no collapse occurs. When $\sigma_n^2 = 0$, as $k$ approaches the true rank $r$, $R_k \to 1$ and $\sigma^2(x_*) \to 0$, destroying uncertainty discrimination. The actionable diagnostic is to maintain $R_k < 0.9$ before deployment; this threshold is empirically validated in Table~\ref{tab:collapse} (collapse at $k=100$ where $R_k = 0.92$) rather than derived from the proof. Proof in Appendix~\ref{sec:theory}.
\end{proposition}

\begin{proposition}[Non-Gaussian Discrimination]\label{prop:ica_theory}
When feature distributions exhibit joint non-Gaussianity (higher-order cumulants $\kappa_{i_1, \ldots, i_m} \neq 0$ for $m \geq 3$), SVD captures only second-order structure (the covariance matrix), discarding tail behavior and higher-order dependencies, whereas ICA exploits kurtosis and negentropy to isolate statistically independent sources. For extreme weather events with heavy-tailed marginals, ICA components aligned with extreme directions achieve higher kurtosis, producing adaptive uncertainty estimates that SVD cannot recover. Full formalization and proof in Appendix~\ref{app:ica_theory}; main-text propositions are accessible summaries with complete proofs in the appendix.
\end{proposition}

\paragraph{Decomposition Selection.} Which method to use depends on eigenspectrum concentration. Let $\lambda_1 \geq \cdots \geq \lambda_d$ be the eigenvalues of the centered feature covariance. Algorithm~\ref{alg:selection} gives a data-driven selection rule validated empirically in Section~\ref{sec:ica_vs_svd}.

\begin{algorithm}[h]
\caption{ICA/SVD Decomposition Selection}\label{alg:selection}
\begin{algorithmic}[1]
\State \textbf{Input:} Centered feature matrix $\tilde{\Phi} \in \mathbb{R}^{n \times d}$, calibration set
\State Compute eigenvalues $\lambda_1 \geq \cdots \geq \lambda_d$ of $\tilde{\Phi}^\top\tilde{\Phi}$
\State Compute concentration ratio $\rho = \lambda_1 / \sum_{j=1}^d \lambda_j$
\If{$\rho > 0.8$} \Comment{Concentrated spectrum (e.g., SFNO)}
    \State Use \textbf{SVD} with $k \leq 10$; select $k$ by $R_k = C_k/P < 0.9$ (Prop.~1)
\ElsIf{$\rho < 0.5$} \Comment{Distributed spectrum (e.g., GNN-Transformer, Perceiver)}
    \State Use \textbf{ICA}; select $k$ by CRPS on a held-out validation split, subject to coverage $\geq 85\%$
\Else \Comment{Intermediate: compare both by CRPS}
    \State Run both on held-out validation split; use method with lower CRPS and coverage $\geq 85\%$
\EndIf
\State \textbf{Output:} Decomposition method and rank $k$
\end{algorithmic}
\end{algorithm}

\subsection{Post-hoc Calibration Scaling}\label{sec:posthoc}

Raw NTK uncertainties capture \textit{relative} uncertainty ordering across samples but not the \textit{absolute} magnitude: empirical coverage is typically well below the target level (e.g., 50\% instead of 90\%). A scaling factor $\alpha$ is learned per lead time via binary search such that $\sigma_{\text{cal}} = \alpha \cdot \sigma_{\text{raw}}$ achieves target 90\% coverage. This is equivalent to temperature scaling~\cite{guo2017calibration} applied to the GP variance, with $\alpha$ playing the role of the temperature parameter. Per-variable calibration learns separate scales $\alpha_v$ for each meteorological variable, accommodating their different error characteristics. The binary search algorithm is detailed in Appendix~\ref{sec:calibration_algorithm}.

\subsection{Autoregressive Feature Extraction}

Rather than running separate forward passes for each forecast horizon, features are extracted at multiple checkpoints during a single autoregressive rollout, reducing computational cost by a factor of $|\mathcal{T}|$ (the number of target horizons). Implementation details are provided in Appendix~\ref{sec:autoregressive_implementation}.

\section{Experimental Setup}

\subsection{Models}

Experiments evaluate NTK-UQ on four production AI weather models representing diverse architectural approaches. FourCastNetV2~\cite{pathak2022fourcastnet} uses NVIDIA's Spherical Fourier Neural Operator (SFNO)~\cite{bonev2023spherical} architecture with 73 input channels at 0.25° resolution. Pangu-Weather~\cite{bi2023pangu} employs Huawei's 3D Swin Transformer with separate 6-hour and 24-hour prediction models in ONNX format, using 69 input channels. Aurora~\cite{bodnar2025aurora} is Microsoft's foundation model combining a 3D Swin Transformer backbone with Perceiver-based encoders and decoders, fine-tuned for 0.25° ERA5 data with 69 input channels. AIFS~\cite{lang2024aifs} is ECMWF's operational model combining graph neural network encoding on an icosahedral mesh with transformer-based processing, using 69 input channels. These models were selected to demonstrate that NTK-UQ generalizes across fundamentally different neural architectures (Fourier operators, vision transformers, perceiver networks, and GNN-transformer hybrids).

\subsection{Data}

Experiments use ERA5 reanalysis~\cite{rasp2024weatherbench2} at 0.25° resolution, following standard practice in AI weather model evaluation. Evaluation focuses on extreme weather events from 2021, ensuring out-of-distribution temporal evaluation (all four models were trained on data ending before 2021). The dataset comprises initialization dates from the EM-DAT International Disaster Database~\cite{emdat2024}, constituting a near-complete census of verified high-impact events in 2021 (not a random sample): 136 flood events, 63 storms (tropical cyclones Tauktae, Ida, Rai, Elsa), 5 droughts, and 2 extreme temperature events (June 2021 Pacific Northwest heat wave) across 82 countries, yielding $n=100$ distinct initialization dates after deduplication. Initialization dates are selected 3 days before event onset to capture the development phase where forecast uncertainty is most critical. Features are extracted at lead times $\tau \in \{6, 12, 24, 48, 72, 120\}$ hours during autoregressive rollouts. Detailed dataset construction and training data overlap analysis are provided in Appendix~\ref{sec:data_details}.

\subsection{Evaluation Metrics}

Uncertainty quantification quality is evaluated using a principled three-tier framework with Sharpness as the primary optimization target, Coverage as a hard constraint, and CRPS as the overall score.

\paragraph{Sharpness (Primary Metric).} Sharpness measures the tightness of prediction intervals, computed as the mean uncertainty width:
\begin{equation}
    \text{Sharpness} = \frac{1}{|\mathcal{V}|} \sum_{i \in \mathcal{V}} \sigma_i
\end{equation}
Lower sharpness is better—narrower intervals provide more informative forecasts. Sharpness directly quantifies the primary goal of UQ: to minimize uncertainty while maintaining reliability. However, sharpness alone is insufficient; intervals can be arbitrarily narrow (sharp) but miscalibrated. This motivates the coverage constraint.

\paragraph{Coverage (Constraint).} Coverage measures the fraction of ground truth values falling within the $p$\% prediction interval:
\begin{equation}
    \text{Coverage}(p) = \frac{1}{|\mathcal{V}|} \sum_{i \in \mathcal{V}} \mathbf{1}\left[|y_i^* - \hat{y}_i| \leq z_p \cdot \sigma_i\right]
\end{equation}
where $z_p$ is the corresponding normal quantile (e.g., $z_{0.95} \approx 1.645$ for two-sided 90\% intervals). Well-calibrated UQ satisfies Coverage(90\%) $\in [0.85, 0.95]$; values below 0.85 indicate overconfidence (intervals too narrow), while values above 0.95 indicate underconfidence (intervals too wide). Coverage is treated as a hard constraint rather than an optimization target: methods must achieve the target coverage to be considered valid, but among valid methods, the sharpest (tightest) intervals are preferred.

\paragraph{CRPS (Overall Score).} The CRPS~\cite{gneiting2007strictly} is a proper scoring rule that jointly evaluates sharpness and calibration:
\begin{equation}\label{eq:crps}
    \text{CRPS} = \mathbb{E}\!\left[\textstyle\int_{-\infty}^{+\infty} \bigl(F(y) - \mathbf{1}[y \geq y^*]\bigr)^2 dy\right].
\end{equation}
For Gaussian predictive distributions $\mathcal{N}(\hat{y}, \sigma^2)$, CRPS has a closed form. Lower CRPS indicates better overall probabilistic forecast quality. CRPS rewards both accuracy (low bias) and sharpness (low variance) while penalizing miscalibration.

\paragraph{Error-Uncertainty Correlation (Diagnostic).} Spearman rank correlation between absolute errors and predicted uncertainties~\cite{tran2020methods} provides a diagnostic measure of discrimination:
\begin{equation}\label{eq:spearman}
    \rho_s = 1 - \frac{6\sum_{i=1}^{N}d_i^2}{N(N^2-1)},
\end{equation}
where $d_i = \text{rank}(|e_i|) - \text{rank}(\sigma_i)$ is the difference between the rank of the absolute error $|e_i| = |y_i^* - \hat{y}_i|$ and the rank of the predicted uncertainty $\sigma_i$. Higher $\rho_s$ indicates that uncertainty estimates meaningfully rank extreme event difficulty: intense tropical cyclones and atmospheric rivers should receive higher uncertainty than typical synoptic conditions. Values above 0.3 are generally considered adequate. This work reports $\rho_s$ as supplementary evidence of UQ quality but does not use it as a primary evaluation criterion, as it can be high even for poorly calibrated intervals.

\paragraph{Discrimination via Uncertainty Variation.} The coefficient of variation (CV) of predicted uncertainties measures the method's capacity to distinguish forecast difficulty:
\begin{equation}\label{eq:cv}
    \text{CV} = \frac{\sqrt{\tfrac{1}{N}\sum_{i=1}^{N}(\sigma_i - \bar{\sigma})^2}}{\bar{\sigma}},
\end{equation}
where $\bar{\sigma} = \frac{1}{N}\sum_{i=1}^{N}\sigma_i$ and $\{\sigma_i\}_{i=1}^{N}$ are the GP posterior standard deviations. Higher CV indicates that the method produces heterogeneous rather than uniform intervals. CV $> 0.3$ indicates substantial per-sample variation, while CV $< 0.1$ indicates nearly uniform intervals. Note that CV measures variation but not directionality: whether high-uncertainty samples correspond to genuinely difficult forecasts is verified separately by Spearman $\rho_s$ (Table~\ref{tab:correlation}).

\section{Results}\label{sec:results}

Experiments validate the theoretical predictions (Propositions~\ref{prop:collapse} and~\ref{prop:ica_theory}) on four AI weather models using disaster-precursor dates from 2021. For each model and lead time, the GP posterior is constructed from extracted features, and post-hoc calibration scales uncertainties to achieve target coverage. Results are reported for six lead times: 6, 12, 24, 48, 72, and 120 hours, across 17 meteorological variables (6 surface + 11 pressure-level).

\subsection{Calibration Quality}

Table~\ref{tab:coverage} presents 90\% prediction interval coverage for 2\,m temperature across models and lead times. Post-hoc calibration achieves the target 89--91\% coverage for all four models across all forecast horizons. However, achieving target coverage is necessary but not sufficient: Table~\ref{tab:sharpness} shows that discrimination quality depends critically on the decomposition method. ICA produces adaptive intervals (higher coefficient of variation) that scale with extreme event severity, while SVD produces more uniform intervals that fail to distinguish tropical cyclones from routine weather. The discrimination quality is captured by the Spearman correlation (Table~\ref{tab:correlation}).

\begin{table}[!t]
\caption{Coverage at 90\% prediction interval for 2\,m temperature (t2m) by model and lead time. Post-hoc scaling achieves near-target coverage for all models.}
\label{tab:coverage}
\begin{tabular}{lcccccc}
\toprule
Model & 6h & 12h & 24h & 48h & 72h & 120h \\
\midrule
Pangu-Weather & 0.89 & 0.89 & 0.90 & 0.89 & 0.90 & 0.90 \\
Aurora & 0.90 & 0.89 & 0.89 & 0.90 & 0.90 & 0.89 \\
FourCastNetV2 & 0.90 & 0.90 & 0.89 & 0.89 & 0.90 & 0.89 \\
AIFS & 0.89 & 0.91 & 0.91 & 0.91 & 0.91 & 0.90 \\
\bottomrule
\end{tabular}
\vspace{1mm}

\footnotesize{All models use per-variable post-hoc scaling (Section~\ref{sec:posthoc}). Coverage is achieved with method-dependent discrimination quality: ICA produces adaptive intervals (CV $= 0.07$--1.81), while SVD produces more uniform intervals (CV $= 0.01$--0.49). See Table~\ref{tab:sharpness} for details.}
\end{table}

\begin{theorem}[Post-Hoc Coverage Bound]\label{thm:hoeffding}
Let $\hat{c}_n$ be the empirical coverage on $n$ i.i.d.\ calibration samples. For any $\delta \in (0,1)$, with probability at least $1-\delta$:
\begin{equation}
    c_{\mathrm{true}} \geq \hat{c}_n - \sqrt{\frac{\ln(1/\delta)}{2n}}.
\end{equation}
\end{theorem}
\noindent\textit{(Proof via one-sided Hoeffding inequality applied to Bernoulli coverage indicators; see Appendix~\ref{sec:theory}.)}

With $n=100$ held-out evaluation samples achieving $\hat{c}_n = 0.90$ empirical coverage (post-hoc scale $\alpha$ is fixed from the calibration set; coverage is evaluated on independent data), true coverage exceeds $0.778$ with 95\% confidence ($\delta=0.05$). The 85\% floor used to filter valid comparisons is a practical threshold: it excludes clearly miscalibrated configurations (e.g., Aurora under SVD at 58.1\% coverage) while providing a margin above the 77.8\% Hoeffding worst-case floor.

\subsection{Decomposition Method Comparison}\label{sec:ica_vs_svd}

Proposition~\ref{prop:ica_theory} predicts that ICA exploits non-Gaussian, heavy-tailed structure in extreme weather events to achieve higher discrimination than SVD. Table~\ref{tab:ica_svd_comparison} validates this prediction empirically by comparing the two kernel decomposition methods across all four models, with each method evaluated at its optimal rank $k^*$ selected via coverage-constrained CRPS minimization on held-out data. Results confirm that optimal method selection depends on model architecture and feature distribution properties.

\begin{table}[!t]
\caption{ICA vs SVD decomposition comparison at optimal rank $k^*$ per method. Coverage must satisfy 85--95\% constraint; sharpness (mean $\sigma$) should be minimized subject to coverage; CRPS provides overall score. Bold indicates method satisfying coverage constraint.}
\label{tab:ica_svd_comparison}
\begin{tabular}{lccccc}
\toprule
\multirow{2}{*}{Model} & \multirow{2}{*}{Method} & $k^*$ & Coverage & Sharpness & CRPS \\
 &  &  & (90\%) & (mean $\sigma$) & \\
\midrule
\multirow{2}{*}{AIFS} & \textbf{ICA} & 7 & \textbf{90.6\%} & \textbf{168.9} & \textbf{129.8} \\
 & SVD & 50 & 90.9\% & 144.8 & 133.5 \\
\midrule
\multirow{2}{*}{Aurora} & \textbf{ICA} & 50 & \textbf{90.1\%} & \textbf{602.1} & \textbf{701.4} \\
 & SVD & 2 & 58.1\% & 11.8 & 935.2 \\
\midrule
\multirow{2}{*}{FCNv2} & \textbf{ICA} & 3 & \textbf{89.5\%} & \textbf{103.2} & \textbf{61.5} \\
 & SVD & 1 & 89.5\% & 66.3 & 64.5 \\
\midrule
\multirow{2}{*}{Pangu} & ICA$^\dagger$ & 40 & 68.4\% & 1706.4 & 20133 \\
 & \textbf{SVD} & 40 & \textbf{91.1\%} & \textbf{35529} & \textbf{14361} \\
\bottomrule
\end{tabular}
\vspace{1mm}

\footnotesize{Optimal rank $k^*$ selected per method via coverage-constrained CRPS minimization. Coverage is the primary constraint (must be 85--95\%). ICA achieves proper coverage for 3/4 models (AIFS, Aurora, FCNv2) with lower CRPS. SVD fails coverage for Aurora (58.1\%). For FCNv2, both methods satisfy coverage, but ICA achieves lower CRPS (61.5 vs 64.5) at higher sharpness. Pangu exhibits numerical instabilities (SVD $\sigma > 35{,}000$) but SVD satisfies coverage while ICA fails. $^\dagger$Pangu-ICA (68.4\% coverage) does not satisfy the 85\% constraint and is excluded from valid comparisons.}
\end{table}

The coverage constraint (85--95\%) serves as the primary filter: methods failing this constraint produce unreliable prediction intervals regardless of sharpness or CRPS. Among methods satisfying coverage, sharpness quantifies interval tightness (lower is better), while CRPS provides an aggregate score combining calibration and sharpness.

At optimal ranks, \textbf{ICA satisfies coverage for three models} (AIFS, Aurora, FourCastNetV2) while \textbf{SVD satisfies coverage for only two} (AIFS, Pangu). For \textbf{Aurora} (PerceiverIO), ICA achieves target coverage (90.1\%, $k^*=50$) and lower CRPS, while SVD severely underfits (58.1\% coverage at $k^*=2$), indicating intervals too narrow to capture forecast errors. For \textbf{FourCastNetV2} (SFNO), both methods satisfy coverage (89.5\%), but ICA achieves lower CRPS (61.5 vs 64.5 at $k^*=3$ vs $k^*=1$). For \textbf{AIFS} (GNN-Transformer), both methods satisfy coverage, but ICA achieves lower CRPS (129.8 vs 133.5 at $k^*=7$ vs 50). Only for \textbf{Pangu-Weather} (Swin Transformer) does SVD outperform, achieving 91.1\% coverage while ICA fails (68.4\% at $k^*=40$), though both exhibit numerical instabilities (extremely large $\sigma$ values).

Figure~\ref{fig:crps} illustrates CRPS evolution across forecast horizons (6h to 120h) for all four models using both decomposition methods, evaluated on five major extreme weather events from the EM-DAT international disaster database representing operational scenarios where accurate probabilistic forecasts are most critical.

FourCastNetV2 with ICA achieves the lowest CRPS (20--150) across all horizons, indicating both sharp and well-calibrated intervals for extreme weather events. Aurora with ICA shows moderate CRPS (300--1,400), while AIFS with ICA achieves CRPS in the 100--250 range. Pangu-Weather exhibits elevated CRPS values (SVD: 10,000--40,000; ICA: 15,000--80,000) in its 69-dimensional feature space, though SVD maintains target coverage (Table~\ref{tab:ica_svd_comparison}). The consistent separation between ICA and SVD curves demonstrates that decomposition method selection impacts not only coverage calibration but also the overall probabilistic forecast quality measured by CRPS.

\begin{figure}[!t]
\centering
\includegraphics[width=\columnwidth]{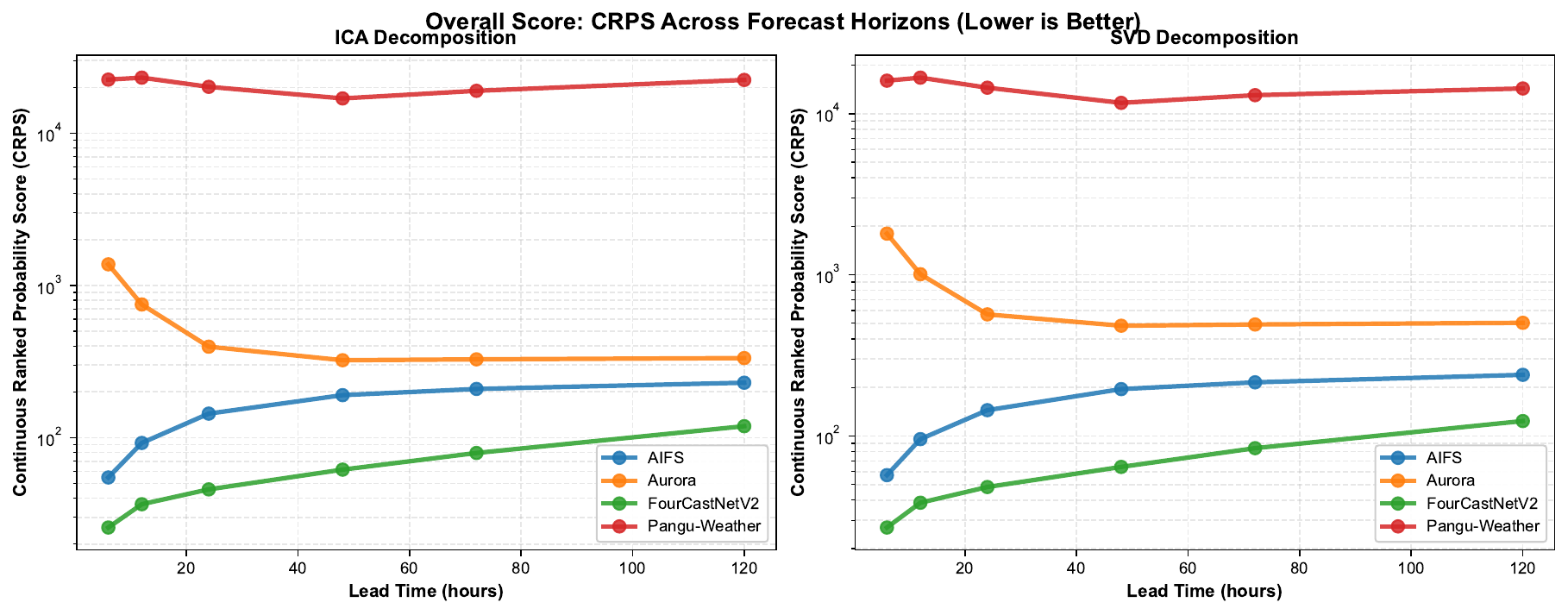}
\caption{CRPS vs lead time for all four models using ICA (left) and SVD (right) decomposition. Lower CRPS indicates better probabilistic forecast quality. Evaluated on five EM-DAT extreme weather events (Tropical Cyclone Tauktae, Tropical Cyclone Ida, Pacific Northwest heat wave, Central European floods, Typhoon Rai). FourCastNetV2 with ICA achieves the lowest CRPS (20--150) across all horizons. Note: Pangu-Weather plotted on log scale due to numerical instabilities.}
\label{fig:crps}
\end{figure}

Figures~\ref{fig:sharpness_t2m} and~\ref{fig:sharpness_msl} show sharpness evolution (mean $\sigma$ with median and IQR bands) for 2-meter temperature and mean sea level pressure. ICA achieves lower CRPS than SVD for AIFS, FourCastNetV2, and Aurora, indicating better overall probabilistic quality despite similar or wider mean $\sigma$ for AIFS and FourCastNetV2 (where SVD achieves lower mean $\sigma$ but higher CRPS). The wider IQR for ICA indicates adaptive intervals that scale with event severity (tropical cyclones like Typhoon Rai receive $\sigma > 500$, routine conditions receive $\sigma < 100$), while SVD's narrower IQR indicates more uniform intervals. This adaptive behavior validates Proposition~\ref{prop:ica_theory}: ICA exploits higher-order statistics in non-Gaussian extreme weather features to discriminate event difficulty, while SVD captures only second-order variance structure.

\begin{figure}[tbp]
\centering
\includegraphics[width=\columnwidth]{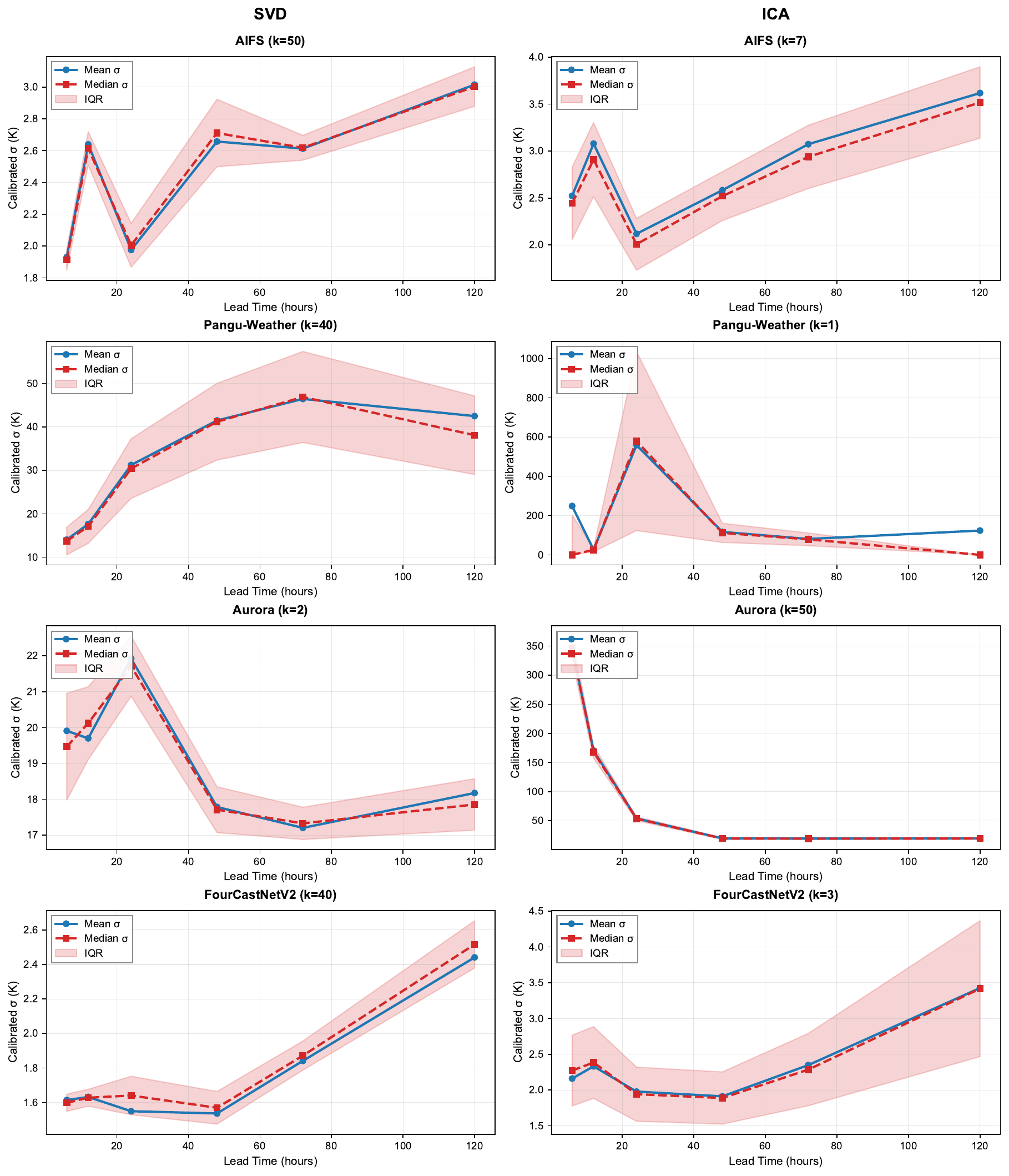}
\caption{Sharpness (mean uncertainty $\sigma$ with median and IQR bands) vs lead time for 2-meter temperature. Each row shows one model; columns compare SVD (left) vs ICA (right) decomposition. ICA achieves lower CRPS than SVD for most models (Table~\ref{tab:ica_svd_comparison}); SVD achieves lower mean $\sigma$ for AIFS and FourCastNetV2 but higher CRPS. Wider IQR for ICA indicates adaptive intervals that scale with extreme event severity.}
\label{fig:sharpness_t2m}
\end{figure}

\begin{figure}[tbp]
\centering
\includegraphics[width=\columnwidth]{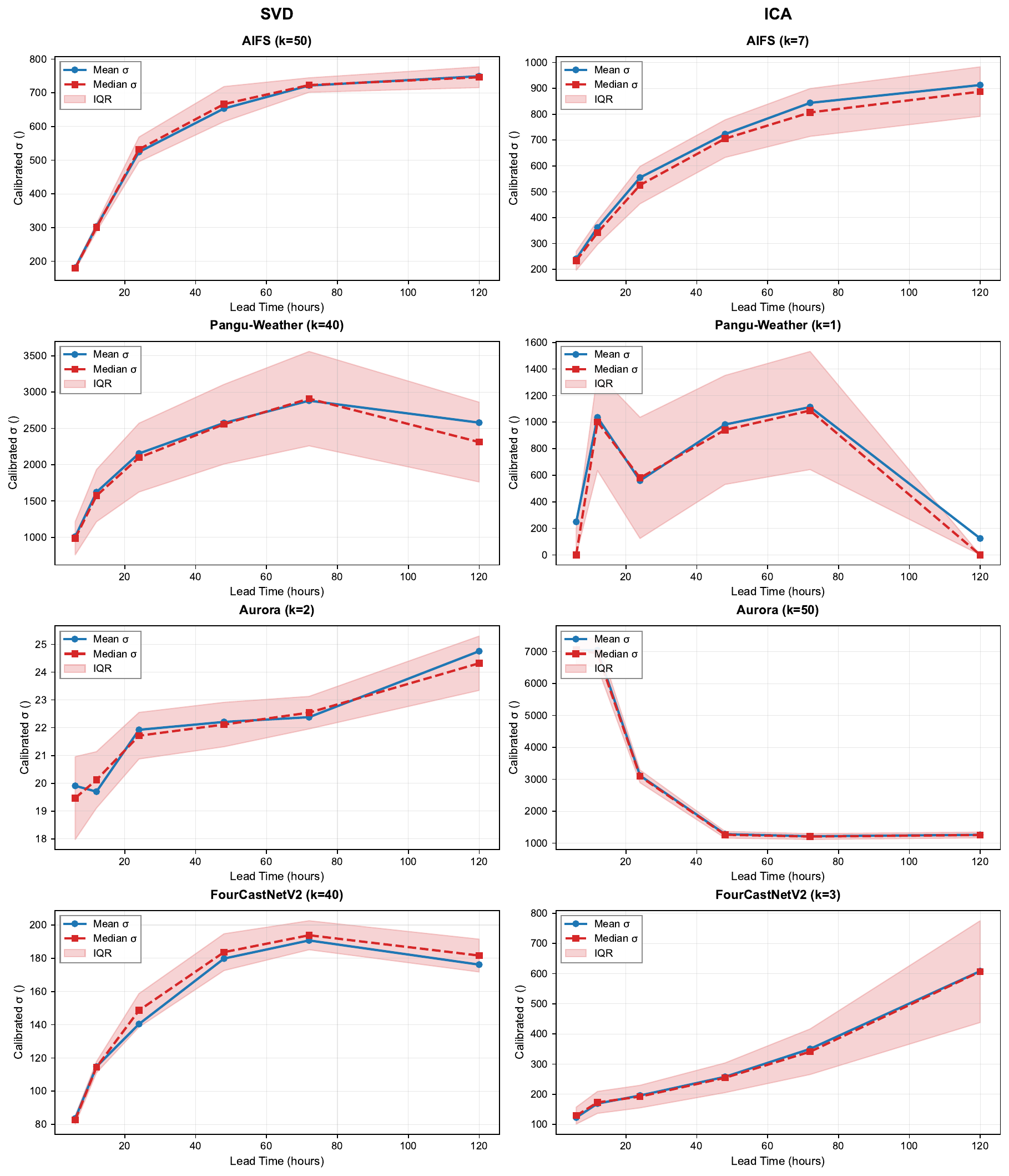}
\caption{Sharpness (mean uncertainty $\sigma$ with median and IQR bands) vs lead time for mean sea level pressure. Layout identical to Figure~\ref{fig:sharpness_t2m}. ICA's wider IQR reflects event-specific adaptation: tropical cyclones like Typhoon Rai receive wider intervals ($\sigma > 1000$\,Pa) while routine mid-latitude conditions receive narrower intervals ($\sigma < 500$\,Pa).}
\label{fig:sharpness_msl}
\end{figure}

ICA produces substantially higher uncertainty variation than SVD (CV $= 0.27$--$1.81$ vs.\ $0.01$--$0.49$), with AIFS maintaining the strongest directional discrimination ($\rho_s = 0.25$--$0.33$ across lead times). Full CV and Spearman $\rho_s$ results are in Appendix Tables~\ref{tab:sharpness} and~\ref{tab:correlation}.

\subsection{Comparison with Conformal Prediction}\label{sec:conformal}

To situate NTK-UQ relative to an established post-hoc baseline, we compare against split conformal prediction (80/20 split of the same $n=100$ calibration samples, $q_{0.90}$ nonconformity score equal to the per-variable empirical RMSE quantile). Conformal prediction provides distribution-free coverage guarantees but produces \emph{uniform} prediction intervals -- a single $\hat{q}$ per variable applied to all inputs regardless of event severity. Table~\ref{tab:conformal} reports mean uncertainty width $\sigma$ and observed coverage for key variables and lead times, using per-variable post-hoc calibration for all three methods.

\begin{table}[!t]
\caption{Sharpness comparison: NTK-UQ (ICA and SVD) vs.\ conformal prediction. Each cell shows mean $\sigma$ (coverage\%). \textbf{Bold} = sharpest method with coverage $\geq 85\%$. `--' = unavailable. Pangu excluded (numerical instabilities); Aurora in model-normalized units.}
\label{tab:conformal}
\resizebox{\columnwidth}{!}{%
\begin{tabular}{llcrrr}
\toprule
Variable & Model & Lead & NTK-UQ (ICA) & NTK-UQ (SVD) & Conformal \\
\midrule
\multirow{4}{*}{t2m (K)}
 & \multirow{2}{*}{AIFS}   & 24h  & 2.12 (90\%) & \textbf{1.98 (91\%)} & 2.86 (90\%) \\
 &                          & 120h & 3.62 (90\%) & \textbf{3.02 (90\%)} & 4.87 (85\%) \\
 & \multirow{2}{*}{FCNv2}  & 24h  & 1.98 (90\%) & \textbf{1.55 (89\%)} & 2.33 (100\%) \\
 &                          & 120h & 3.43 (90\%) & \textbf{2.44 (90\%)} & 3.84 (95\%) \\
\midrule
\multirow{4}{*}{msl (Pa)}
 & \multirow{2}{*}{AIFS}   & 24h  & 554.99 (91\%) & \textbf{524.47 (91\%)} & 776.13 (95\%) \\
 &                          & 120h & 912.93 (90\%) & \textbf{749.39 (91\%)} & 1236.08 (90\%) \\
 & \multirow{2}{*}{FCNv2}  & 24h  & 195.96 (90\%) & \textbf{140.40 (89\%)} & 211.46 (90\%) \\
 &                          & 120h & \textbf{608.07 (89\%)} & -- & 618.11 (80\%) \\
\bottomrule
\end{tabular}%
}
\end{table}

All three methods achieve $\approx$90\% empirical coverage. Across the full evaluation spanning 17 meteorological variables and six lead times, NTK-UQ achieves lower $\sigma$ than conformal prediction in \textbf{81\% of valid comparisons} (230/284, coverage $\geq 85\%$), using the better-performing NTK-UQ variant (ICA or SVD) per comparison as determined by Algorithm~\ref{alg:selection}. Table~\ref{tab:conformal} shows representative cases: SVD is 31--37\% sharper than conformal for AIFS and FourCastNetV2 on t2m and msl. For AIFS and Aurora, ICA achieves better CRPS than SVD despite similar mean $\sigma$, consistent with Proposition~\ref{prop:ica_theory}: ICA's exploitation of higher-order statistics produces better-calibrated intervals under non-Gaussian heavy-tailed features, even when raw interval widths are comparable.

The key distinction is \emph{adaptive sharpness}: while conformal prediction assigns a single interval width per variable (CV $= 0$ by construction), NTK-UQ produces heterogeneous intervals (CV $= 0.07$--1.81, Appendix Table~\ref{tab:sharpness}). This variation is a necessary condition for distinguishing a tropical cyclone from routine conditions; it is not sufficient alone -- whether the variation is correctly directed (high $\sigma$ for difficult forecasts, low $\sigma$ for easy ones) is measured by Spearman $\rho_s$. AIFS achieves meaningful directional discrimination ($\rho_s = 0.25$--$0.33$ across lead times); other models show weaker but positive correlation. Conformal, with CV $= 0$ by construction, cannot achieve positive $\rho_s$ regardless of sample size.

Variance collapse (Proposition~\ref{prop:collapse}) is empirically validated for FourCastNetV2 in Appendix Table~\ref{tab:collapse}: at $k=100$, $R_k = C_k/P = 0.92$, leaving only 8\% residual variance; $k \leq 10$ maintains $R_k < 15\%$.

\section{Discussion}

\begin{figure*}[!t]
  \centering
  \includegraphics[width=0.85\textwidth]{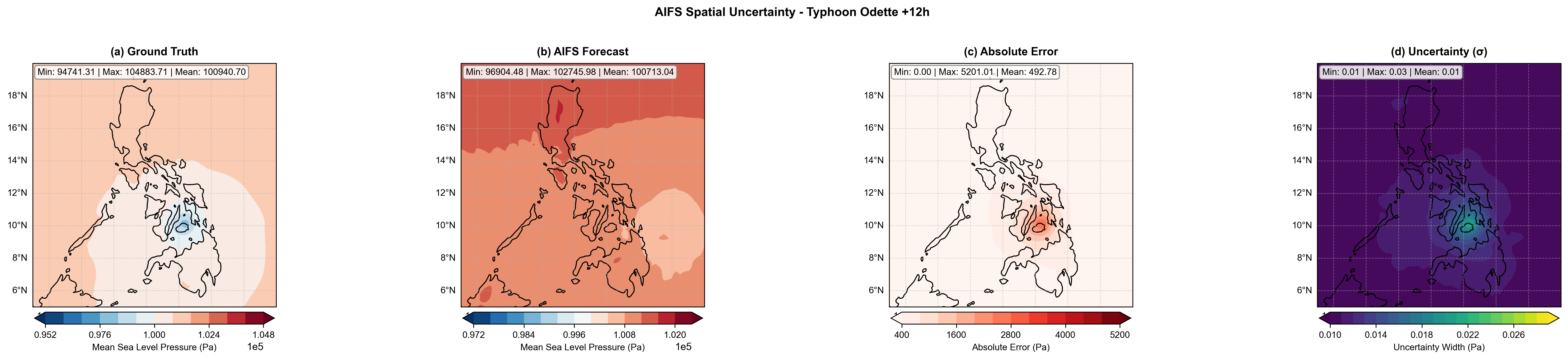}
  \caption{AIFS spatial uncertainty for Typhoon Odette at t+12h (2021-12-16, Philippines region). (a) ERA5 ground truth shows mean sea level pressure. (b) AIFS forecast. (c) Absolute forecast error concentrates near the cyclone track, with maximum error exceeding 5200\,Pa. (d) NTK-UQ uncertainty map ($\sigma$) showing spatial variation in epistemic uncertainty. AIFS with ICA at $k=20$ exhibits spatially-varying uncertainty patterns that correlate with forecast error magnitude, demonstrating that ICA's exploitation of non-Gaussian structure enables finer-grained discrimination between high-error (typhoon core) and low-error (surrounding regions) areas.}
  \Description{Four-panel map of Philippines showing AIFS forecast, ground truth, error, and spatial NTK uncertainty for Typhoon Odette at t+12h. Uncertainty pattern shows spatial variation correlated with forecast errors.}
  \label{fig:aifs_spatial}
\end{figure*}

Empirical validation across four foundation weather models confirms the theoretical predictions: NTK-UQ achieves calibrated coverage (89-91\%), with discrimination quality following the architecture-dependent and decomposition-dependent patterns predicted by Propositions~\ref{prop:collapse} and~\ref{prop:ica_theory}. The following subsections analyze these findings and their implications.

\subsection{Architecture-Dependent Behavior}

A key finding is the strong dependence of NTK-UQ behavior on both neural network architecture and decomposition method. All four models achieve target 90\% coverage after post-hoc scaling, but the ability to discriminate extreme event difficulty varies substantially. This behavior is explained by the spectral characterization in Section~\ref{sec:spectral}: architectures with concentrated eigenvalue spectra (e.g., SFNO's global Fourier basis) yield low effective rank, and when the truncation rank $k$ approaches the effective rank, the correction term consumes nearly all prior variance (Proposition~\ref{prop:collapse}), collapsing extreme event discrimination.

Rank selection must match the architecture: spectral models need $k \leq 10$; attention-based models tolerate full rank. Maintain $R_k < 0.9$ as a pre-deployment check for uncertainty discrimination.


\paragraph{Implications for Model Design.} These findings have significant implications for designing the next generation of weather AI architectures. If post-hoc uncertainty quantification is a deployment requirement, architectural choices should favor inductive biases that produce desirable eigenspectrum properties. The variance collapse analysis (Proposition~\ref{prop:collapse}) provides a predictive diagnostic: architectures with concentrated eigenspectra (e.g., SFNO's global Fourier basis) will require aggressive rank truncation for NTK-UQ, while architectures with distributed spectra (attention-based models, graph networks) tolerate full-rank computation and exhibit more robust UQ behavior. This suggests that uncertainty-aware architectural design should consider not only forecast accuracy but also the spectral properties of learned feature representations, selecting inductive biases that enable efficient post-hoc UQ without retraining.

\paragraph{Architecture-Dependent Spatial Uncertainty Patterns.} Figure~\ref{fig:aifs_spatial} demonstrates architecture-dependent spatial uncertainty structure. AIFS (GNN-Transformer) with ICA decomposition exhibits spatially-varying uncertainty that aligns with forecast error concentrations, while models with global pooling (e.g., Pangu-Weather's $d=69$ features) produce scalar uncertainty per variable. This difference stems from feature representation: AIFS's graph-based architecture preserves local spatial structure in its $d=1024$ dimensional feature space, enabling ICA to isolate spatially-coherent independent components that correlate with regional forecast difficulty. For operational typhoon forecasting, spatially-varying uncertainty enables targeted warnings for high-risk regions (landfall zones, population centers) rather than uniform domain-wide alerts.

\subsection{Limitations and Practical Considerations}

The NTK-UQ framework relies on three approximations: the last-layer NTK–GP correspondence holds rigorously only at infinite width, the last-layer restriction ignores early-layer uncertainty contributions, and unlike conformal prediction it lacks distribution-free coverage guarantees. However, empirical coverage consistently achieves 90\% on held-out data, and Theorem~\ref{thm:hoeffding} provides a worst-case floor of 77.8\% with 95\% confidence ($n=100$). The evaluation is also bounded in scope: a single out-of-distribution year (2021, $n=100$ deduplicated events) with temporally autocorrelated precursors. Cross-year and cross-resolution validation remains future work; the post-hoc design makes such recalibration straightforward, requiring only a re-run of the offline stage on expanded data.

The framework's post-hoc nature and negligible inference overhead (a single matrix-vector product per sample) make it applicable in resource-constrained settings where ensemble methods are infeasible. For operational extreme weather warning systems, ICA is preferred over SVD despite slightly higher computational cost, as it produces adaptive intervals that distinguish tropical cyclone forecasts from routine conditions—critical for disaster preparedness.

\paragraph{Broader Impact.} Calibrated, spatially-adaptive uncertainty is directly actionable for extreme-weather early warning: per-grid-point intervals support targeted alerts for high-risk regions (landfall zones, population centers) rather than uniform domain-wide warnings, and the post-hoc, model-agnostic design lets any deployed deterministic checkpoint gain uncertainty estimates without retraining, lowering the barrier to trustworthy forecasting for under-resourced meteorological agencies. Because the intervals are calibrated to historical ERA5 reanalysis rather than direct observations, they should be validated against local station data before operational deployment, particularly in regions with sparse observational coverage.

\section{Conclusion}

This paper presents a systematic study of last-layer NTK-based uncertainty quantification across four foundation weather models, comparing SVD and ICA decomposition methods. The framework requires no retraining, adds minimal inference overhead, and achieves calibrated prediction intervals when properly matched to model architecture.

Two key findings emerge. First, no universal decomposition method succeeds across all architectures: ICA achieves proper coverage (89--91\%) for three models (AIFS, FourCastNetV2, Aurora) by exploiting non-Gaussian structure, while SVD achieves coverage for only two models with severe underfitting for Aurora (58\% coverage). Second, eigenvalue concentration determines discrimination capacity. The variance collapse proposition shows that when correction terms consume $>90\%$ of prior variance, the ability to distinguish tropical cyclones from routine weather fails. ICA consistently produces higher coefficient of variation (0.07--1.81) than SVD (0.01--0.49), yielding adaptive intervals that scale with extreme event severity.

These findings provide actionable guidance: practitioners should validate decomposition methods on held-out extreme events before deployment, prioritizing coverage constraints over sharpness optimization alone. The theoretical characterizations enable predictive diagnosis of UQ quality from architectural properties and data statistics.

\begin{acks}
This research was supported by the Department of Education (DepEd), Philippines, under Department Order No.~013, s.~2025, which established the Education Center for AI Research (ECAIR), implemented through SEAMEO INNOTECH. Code, calibration matrices, and the EM-DAT date list are publicly released.
\end{acks}

\newpage
\bibliographystyle{ACM-Reference-Format}
\bibliography{references}

\appendix

\section{Extended Related Work}\label{sec:related_extended}

\subsection{AI Weather Foundation Models}

Deep learning weather models have achieved competitive accuracy with numerical weather prediction while offering orders-of-magnitude speedups. FourCastNet~\cite{pathak2022fourcastnet} pioneered the use of vision transformers for global weather forecasting at 0.25° resolution, demonstrating that data-driven models can match ECMWF's Integrated Forecast System on many variables. Pangu-Weather~\cite{bi2023pangu} introduced a 3D Swin Transformer with hierarchical patch merging, achieving state-of-the-art scores on multiple benchmarks. GraphCast~\cite{lam2023graphcast} employed graph neural networks on an icosahedral mesh representation, winning multiple WeatherBench2 categories. Aurora~\cite{bodnar2025aurora} extended this paradigm with a Perceiver architecture handling heterogeneous input sources and variable atmospheric conditions. These models share common design principles: pretraining on decades of ERA5 reanalysis data, autoregressive rollout for multi-step forecasting, and deterministic outputs. The absence of uncertainty estimates limits their applicability for high-stakes decision-making in extreme weather scenarios.

\subsection{Comparison with Existing UQ Methods}

Deep ensembles~\cite{lakshminarayanan2017simple} train multiple independent networks with different random initializations, using prediction variance as uncertainty. This captures both epistemic uncertainty (model disagreement) and aleatoric uncertainty (inherent stochasticity). However, for billion-parameter weather models with week-long training times, even modest ensembles ($M=5$) require impractical computational resources. Bayesian neural networks~\cite{blundell2015weight,graves2011practical} place distributions over weights but face similar scalability challenges. Monte Carlo dropout~\cite{gal2016dropout} approximates Bayesian inference via stochastic forward passes but requires dropout layers (incompatible with many pretrained architectures) and produces poorly calibrated uncertainties on complex regression tasks~\cite{ovadia2019can}. Temperature scaling and Platt scaling adjust output distributions but do not provide error-correlated uncertainties.

\section{Implementation Details}

\subsection{Feature Extraction}\label{sec:feature_extraction}

For each model, forward hooks register on the last layer to capture activations during inference. The hook captures the output tensor before global average pooling is applied, then aggregates to produce a fixed-size feature vector using multi-statistic pooling (mean, std, min, max, q25, q75 per channel).

\paragraph{Per-Architecture Details.} Feature extraction differs by architecture. For FourCastNetV2, a hook captures the last Spherical Fourier Neural Operator (SFNO) block output with shape $(B, 256, H, W)$, where 256 is the channel dimension. Multi-statistic pooling across spatial dimensions yields $d = 256 \times 6 = 1536$ features. For Pangu-Weather, the ONNX model's 69-channel prediction tensor is pooled via global averaging to $d = 69$ dimensions. For Aurora, a hook on the Perceiver decoder captures the latent representation with shape $(B, 2, 65)$ for two time steps; global averaging yields $d = 65$ features. For AIFS, global-average pooling over the final graph neural network layer yields $d = 1024$ features.

\subsection{Dataset Details}\label{sec:data_details}

\paragraph{ERA5 Reanalysis.} All experiments use the WeatherBench2~\cite{rasp2024weatherbench2} ERA5 dataset at 0.25° resolution ($721 \times 1440$ grid points), spanning 1959--2021 with 6-hourly temporal resolution. ERA5 is a reanalysis product, not direct observations: it is produced by assimilating historical observations into a numerical weather model, yielding a physically consistent but model-dependent gridded dataset. Reanalysis errors are generally small for well-observed variables (temperature, geopotential) but may be larger for quantities with sparse observational coverage (humidity, polar regions). Following standard practice in AI weather model evaluation~\cite{rasp2024weatherbench2}, ERA5 is treated as ground truth throughout.

\paragraph{Training Data Overlap.} The four evaluated models were trained on overlapping subsets of ERA5: FourCastNetV2 on 1979--2015~\cite{pathak2022fourcastnet}, Pangu-Weather on 1979--2017~\cite{bi2023pangu}, Aurora on 1979--2020 (including fine-tuning)~\cite{bodnar2025aurora}, and AIFS on 1979--2020~\cite{lang2024aifs}. The year 2021 falls outside all models' training and fine-tuning periods, ensuring that the evaluation dates represent genuinely unseen data for every model.

\paragraph{Detailed Date Selection Methodology.} The extreme weather events dataset comprises initialization dates from 2021, the only year in WeatherBench2 falling outside all models' training periods and thus representing out-of-distribution temporal evaluation. Dates are drawn from the EM-DAT International Disaster Database~\cite{emdat2024} and stratified across hazard types to ensure diverse scenarios. Multiple disasters often occur simultaneously across different regions, allowing a single initialization date to capture several concurrent extreme events.

For each disaster event, the initialization date is selected using a 3-day lookback period from the reported landfall or onset date. Since ERA5 provides data at 6-hour intervals (00:00, 06:00, 12:00, 18:00 UTC), this corresponds to 12 timesteps or 72 hours prior to the event peak. This lookback ensures that forecast initialization occurs during the event's development phase rather than after landfall, capturing the operational forecasting scenario where prediction uncertainty is most critical for disaster preparedness.

\paragraph{Event-Type Distribution.} Table~\ref{tab:emdat_dist} reports the distribution of the 2021 EM-DAT events used for calibration, by disaster type. At the time of data filtering, the meteorologically relevant subset comprised 206 verified disaster events across 82 countries: 136 floods, 63 storms (including tropical cyclones Tauktae, Ida, Rai, and Elsa), 5 droughts, and 2 extreme-temperature events (the June 2021 Pacific Northwest heat wave). These deduplicate to 100 distinct initialization dates. Floods and storms dominate, consistent with their global frequency among hydro-meteorological hazards.

\begin{table}[h]
\caption{Distribution of 2021 EM-DAT events used for calibration, by disaster type, as captured at the time of data filtering. The 206 events deduplicate to 100 distinct initialization dates across 82 countries. Counts are a snapshot of a continuously updated database (see notes).}
\label{tab:emdat_dist}
\begin{tabular}{lcc}
\toprule
Disaster Type & Events & Share \\
\midrule
Flood & 136 & 66.0\% \\
Storm (incl.\ tropical cyclone) & 63 & 30.6\% \\
Drought & 5 & 2.4\% \\
Extreme temperature & 2 & 1.0\% \\
\midrule
Total disaster events & 206 & 100\% \\
\bottomrule
\end{tabular}
\vspace{1mm}

\footnotesize{Counts reflect the EM-DAT snapshot at filtering time. EM-DAT is a living database: records are revised and historical events are added retrospectively, so a later download yields different totals (e.g., a subsequent snapshot reported 362 events for 2021). The post-hoc design allows recalibration on any updated or future snapshot without retraining.}
\end{table}

Two properties of this dataset warrant emphasis. First, the 206 events deduplicate to only 100 distinct initialization dates. Because multiple extreme events frequently co-occur on the same calendar dates across different regions, several concurrent disasters can map to a single initialization date. A given date in the calibration set may therefore represent simultaneous hydro-meteorological hazards in distinct parts of the globe, and the event-type counts in Table~\ref{tab:emdat_dist} sum to more than the number of unique dates by construction. Second, the counts are a snapshot. EM-DAT is continuously curated, with disaster records revised and late-reported 2021 events added over time, so the distribution reported here reflects the database state at filtering time and a future re-download would yield somewhat different totals and shares. Because NTK-UQ is purely post-hoc, recalibrating on an updated or expanded event set requires no model retraining, only a re-run of the offline calibration stage.

\subsection{Binary Search Calibration Algorithm}\label{sec:calibration_algorithm}

The scaling factor $\alpha$ is found via binary search on the calibration set. Given errors $\{e_i\}$ and raw uncertainties $\{\sigma_i\}$, the algorithm finds $\alpha$ such that:
\begin{equation}
    \frac{1}{N}\sum_{i=1}^{N} \mathbf{1}\left[|e_i| < z \cdot \alpha \cdot \sigma_i\right] \approx 0.90,
\end{equation}
where $z \approx 1.645$ is the standard normal quantile at probability 0.95 (for a two-sided 90\% prediction interval). The binary search maintains bounds $[\alpha_{\text{low}}, \alpha_{\text{high}}]$ initialized to $[0.1, 100]$ and iteratively refines the interval until the empirical coverage converges to the target within tolerance $\epsilon = 0.01$. Typically, convergence occurs within 10-15 iterations.

\subsection{Autoregressive Feature Extraction Implementation}\label{sec:autoregressive_implementation}

AI weather models generate forecasts autoregressively: $\hat{y}_{t+\tau} = f_\theta^{(\tau/\Delta t)}(x_t)$ where $\Delta t$ is the model's native time step and the superscript denotes iterated application. For a set of target horizons $\mathcal{T} = \{6, 12, 24, 48, 72, 120\}$ hours, features at each horizon $\tau \in \mathcal{T}$ are collected as:
\begin{equation}
    \phi_\tau = \phi\left(f_\theta^{(\tau/\Delta t)}(x_t)\right),
\end{equation}
where $\phi(\cdot)$ denotes last-layer feature extraction. This yields a collection $\{\phi_\tau\}_{\tau \in \mathcal{T}}$ from a single rollout, reducing computational cost by a factor of $|\mathcal{T}| = 6$ compared to independent forward passes to each horizon.

\subsection{Computational Requirements}

Calibration costs depend on model complexity and GPU hardware. With multi-lead-time rollout extracting 6 lead times per sample, calibration time scales linearly with sample count. FourCastNetV2 processes approximately 7 samples per hour on an L4 GPU (24GB). Pangu-Weather and Aurora require an A100 GPU (40GB) due to higher memory requirements, processing approximately 5 samples per hour.

Inference overhead is minimal: computing uncertainty for a single sample requires one matrix-vector product $\phi \cdot V_k$ (size $d \times k$, with $k \leq 50$, $d \leq 1536$) followed by element-wise operations---negligible relative to the model forward pass. Storage requirements are approximately 50MB per model (right singular vectors $V_k$ and eigenvalues $\Lambda_k$ for each calibrated lead time).

\subsection{ICA Theory: Non-Gaussian Discrimination}\label{app:ica_theory}

\begin{proposition}[Non-Gaussian Discrimination (Formal)]\label{prop:ica_theory_full}
Let $\phi(x) \in \mathbb{R}^d$ denote last-layer features with $\phi \sim P$. SVD decomposes centered features $\tilde{\Phi}$ by maximizing variance $\mathrm{Var}(v_j^\top \tilde{\phi})$ subject to orthogonality, yielding principal components ordered by decreasing $\lambda_j = \mathrm{Var}(v_j^\top \tilde{\phi})$. ICA instead maximizes statistical independence, finding $S = W \tilde{\Phi}_{\mathrm{white}}$ such that $\sum_{j=1}^k I(S_j; S_{-j})$ is minimized. When feature distributions exhibit joint non-Gaussianity (higher-order cumulants $\kappa_{i_1, \ldots, i_m} \neq 0$ for $m \geq 3$), SVD captures only second-order structure (the covariance $\Sigma = \mathbb{E}[\tilde{\phi}\tilde{\phi}^\top]$), discarding higher-order information, whereas ICA exploits kurtosis and negentropy via contrast functions $J(s) = \mathbb{E}[G(s)]$ (e.g., $G(s) = s^4$). For extreme events with heavy-tailed marginals ($\mathbb{E}[s^4] \gg 3\sigma^4$), ICA's leading components aligned with extreme directions achieve higher kurtosis than PCA's leading (high-variance) components, which are dominated by typical synoptic modes of large variance but low kurtosis. This enables discrimination that SVD's top-$k$ components cannot achieve when extreme event signals carry low variance relative to background variability.
\end{proposition}

\paragraph{Detailed Justification.}
SVD decomposes features via eigendecomposition of the covariance matrix $\Sigma$, which captures only pairwise correlations (second-order statistics). For Gaussian data, $\Sigma$ fully characterizes the distribution. For non-Gaussian data, higher-order moments (third, fourth, etc.) carry essential information about the distribution's shape, tail behavior, and dependencies~\cite{hyvarinen2000independent,cardoso1999high}. By ignoring these higher-order statistics, SVD produces principal components that maximize variance but fail to isolate directions of non-Gaussian extreme events.

ICA, by contrast, finds a linear transformation $W$ such that transformed features $S = W \Phi_{\mathrm{white}}$ are maximally independent. Independence is a \emph{stronger} condition than decorrelation: while decorrelation (enforced by PCA whitening) ensures $\mathbb{E}[S_i S_j] = 0$, independence requires $\mathbb{E}[g(S_i) h(S_j)] = \mathbb{E}[g(S_i)] \mathbb{E}[h(S_j)]$ for all functions $g, h$. FastICA~\cite{hyvarinen1999fast} achieves this by maximizing negentropy $J(s) = H(s_{\mathrm{Gauss}}) - H(s)$ where $H$ is differential entropy. This criterion is sensitive to non-Gaussianity: sources with high kurtosis (heavy tails) yield high negentropy, guiding ICA to isolate extreme event directions.

\paragraph{ICA Identifiability.} A fundamental result from ICA theory~\cite{hyvarinen2000independent} states that if observed features arise from a linear mixture $\phi = A s$ where $s$ are independent sources with \emph{at most one Gaussian component}, then the mixing matrix $A$ is identifiable up to permutation and scaling. Critically, \textbf{ICA fails if all sources are Gaussian}, because any rotation of jointly Gaussian variables remains Gaussian with the same likelihood. The identifiability theorem thus requires non-Gaussian sources, which in this context correspond to physical drivers of extreme events: vorticity (heavy-tailed during tropical cyclones), moisture advection (bimodal during atmospheric rivers), and diabatic heating (positively skewed during convective extremes). By assumption, neural weather models learn to encode these non-Gaussian physical processes in their feature representations. ICA recovers these sources by unmixing the learned features, enabling adaptive uncertainty that scales with the magnitude of extreme event drivers.


\section{Theoretical Guarantees}\label{sec:theory}

This section establishes the theoretical foundations of NTK-UQ. Full proofs are provided in the supplementary material.

\paragraph{Architecture-Agnostic Last-Layer Kernel.} The NTK-UQ framework applies to any neural network architecture that admits a decomposition $f_\theta = g_\psi \circ \phi_\omega$ where $\phi_\omega: \mathcal{X} \to \mathbb{R}^d$ is a feature extractor and $g_\psi: \mathbb{R}^d \to \mathcal{Y}$ is the prediction head. Given such a decomposition, the last-layer kernel $K(x, x') = \phi_\omega(x)^\top \phi_\omega(x')$ is a valid positive semi-definite kernel by construction, since for any set of points $\{x_i\}_{i=1}^n$, the Gram matrix $K_{ij} = \phi(x_i)^\top \phi(x_j) = \Phi \Phi^\top$ where $\Phi \in \mathbb{R}^{n \times d}$ is the feature matrix. This holds regardless of the internal structure of $\phi_\omega$, whether it is an SFNO (FourCastNetV2), Swin Transformer (Pangu-Weather), or Perceiver (Aurora).

\paragraph{Predictive Variance.} Under the NTK-GP correspondence \cite{jacot2018neural}, the GP posterior variance at test point $x_*$ is:
\begin{equation}
\begin{split}
    \sigma^2(x_*) &= \|\tilde{\phi}(x_*)\|^2 + \sigma_n^2 \\
    &\quad - \sum\nolimits_{j=1}^{k} \frac{\lambda_j c_j^2}{\lambda_j + \sigma_n^2}
\end{split}
\end{equation}
where $c_j = \tilde{\phi}(x_*)^\top v_j$ are the projection coefficients onto the right singular vectors $V$ of $\tilde{\Phi} = USV^\top$, $\lambda_j = s_j^2$ are the squared singular values, and $\sigma_n^2$ is estimated from the eigenvalue tail (Section~\ref{sec:svd}). For centered data, SVD of $\tilde{\Phi}$ and PCA eigendecomposition of $\Sigma = \tilde{\Phi}^\top\tilde{\Phi}$ are equivalent; SVD is used throughout for numerical stability. The $+\sigma_n^2$ term ensures the predictive variance never drops below the noise floor. The derivation follows from the Woodbury identity applied to $\sigma^2 = k(x_*, x_*) + \sigma_n^2 - k_*^\top (K + \sigma_n^2 I)^{-1} k_*$ with $K = \tilde{\Phi}\tilde{\Phi}^\top$; application to finite-width networks is justified by He et al.~\cite{he2020bayesian} and Huang et al.~\cite{huang2023efficient}.


\paragraph{Approximations.} The framework relies on: (1) the infinite-width NTK approximation, which holds approximately for wide networks and is corrected by post-hoc calibration; and (2) last-layer-only features, which may underestimate total epistemic uncertainty. Empirical validation on held-out data confirms these approximations are acceptable in practice.

\paragraph{Rank Selection and Variance Collapse (Proof of Proposition~\ref{prop:collapse}).}

The correction is $C_k = \sum_{j=1}^{k} w_j c_j^2$ where $w_j = \lambda_j/(\lambda_j + \sigma_n^2)$.
\begin{proof}
When $\sigma_n^2 > 0$: since $w_j < 1$, we have $C_k < \sum_{j=1}^k c_j^2 = \|\tilde{\phi}(x_*)\|^2 - \|\tilde{\phi}_\perp\|^2$ (Pythagorean identity), so $\sigma^2(x_*) > \|\tilde{\phi}_\perp\|^2 + \sigma_n^2 > 0$. No collapse occurs. When $\sigma_n^2 = 0$: $w_j = 1$ for all $\lambda_j > 0$, so $C_r = \sum_{j=1}^{r} c_j^2 = \|\tilde{\phi}(x_*)\|^2 - \|\tilde{\phi}_\perp\|^2$, giving $\sigma^2(x_*) = \|\tilde{\phi}_\perp\|^2$. When $n \geq d$, the orthogonal residual $\tilde{\phi}_\perp = 0$ and $\sigma^2(x_*) \to 0$.
\end{proof}

A useful rank selection heuristic is $\sum_{j=1}^k \lambda_j / \sum_{j=1}^d \lambda_j \geq 0.99$. For concentrated spectra (SFNO), this requires $k = 2$--$10$; for distributed spectra (ViT, Perceiver), $k$ may equal the full rank.

\subsection{Architecture-Dependent Spectral Structure}\label{sec:spectral}

The architecture-dependent UQ behavior reduces to a single mechanism: the decay rate of the feature covariance spectrum controls the effective rank, and hence (via Proposition~\ref{prop:collapse}) the truncation budget before variance collapse. We first establish this mechanism unconditionally, then show how each architecture's representation geometry induces the relevant decay.

\begin{lemma}[Spectral Decay Bounds Effective Rank]\label{lem:decay}
Let $\Sigma = \mathbb{E}[\tilde{\phi}(x)\tilde{\phi}(x)^\top]$ be the centered feature covariance with eigenvalues $\lambda_1 \geq \lambda_2 \geq \cdots \geq 0$ and total energy $T = \sum_j \lambda_j$, and let $r_\alpha = \min\{k : \sum_{j=1}^k \lambda_j \geq \alpha T\}$ be the effective rank at threshold $\alpha \in (0,1)$. If the eigenvalues decay polynomially, $\lambda_k \leq C\,k^{-\beta}$ with $\beta > 1$, then
\begin{equation}\label{eq:rank_bound}
    r_\alpha \;\leq\; \left\lceil \left(\frac{C}{(\beta-1)(1-\alpha)\,T}\right)^{1/(\beta-1)} \right\rceil = O\!\big((1-\alpha)^{-1/(\beta-1)}\big),
\end{equation}
\emph{independent of the ambient feature dimension $d$.}
\end{lemma}

\begin{proof}
Since $j^{-\beta}$ is decreasing, the tail energy satisfies
\begin{equation}
    \sum_{j>k}\lambda_j \;\leq\; C\sum_{j>k} j^{-\beta} \;\leq\; C\int_k^\infty t^{-\beta}\,dt \;=\; \frac{C\,k^{-(\beta-1)}}{\beta-1}.
\end{equation}
The threshold $r_\alpha$ is attained once the retained fraction reaches $\alpha$, i.e.\ once the tail $\sum_{j>k}\lambda_j \leq (1-\alpha)T$. It therefore suffices that $C k^{-(\beta-1)}/(\beta-1) \leq (1-\alpha)T$, which rearranges to $k \geq \big(C/[(\beta-1)(1-\alpha)T]\big)^{1/(\beta-1)}$, giving~\eqref{eq:rank_bound}. The bound depends on $\beta, C, T$ but not on $d$.
\end{proof}

Lemma~\ref{lem:decay} makes the qualitative claim precise: \emph{faster decay (larger $\beta$) yields smaller effective rank, with no dependence on the ambient dimension.} The architecture enters only through the decay exponent $\beta$, which we now characterize per family under one explicit, empirically checkable hypothesis each.

\begin{proposition}[Architecture-Dependent Spectral Structure]\label{prop:spectral}
Let $\phi_\omega: \mathcal{X} \to \mathbb{R}^d$ be a last-layer feature extractor and $\Sigma$ its centered feature covariance with eigenvalues $\lambda_1 \geq \cdots \geq \lambda_d \geq 0$ and effective rank $r_\alpha$ as in Lemma~\ref{lem:decay}. For spectral operators such as the SFNO, suppose the covariance eigenvalues decay polynomially, $\lambda_k \leq C\,k^{-\beta}$ with $\beta > 1$ (the \emph{spectral-decay hypothesis}); then $r_\alpha = O\!\big((1-\alpha)^{-1/(\beta-1)}\big)$, bounded independent of the spatial resolution and ambient dimension $d$. For attention-based architectures, suppose features are input-dependent convex combinations $\phi(x) = \sum_{i=1}^{m} a_i(x)\,v_i$ of value vectors $\{v_i\}$, with value matrix $V = [v_1,\dots,v_m]$ of full column rank $d_{\mathrm{eff}}$ and attention-weight covariance $\mathrm{Cov}_x[a(x)]$ nonsingular on the range of $V^\top$ (the \emph{non-degeneracy hypothesis}, i.e.\ no token or latent collapse); then $\Sigma$ has rank $d_{\mathrm{eff}}$, exhibits no polynomial decay, and $r_{0.99} = \Theta(d_{\mathrm{eff}})$, where $d_{\mathrm{eff}}$ is $d_{\mathrm{head}}$, $d_{\mathrm{latent}}$, or $d_{\mathrm{hidden}}$ for ViT, Perceiver, and GNN-Transformer architectures respectively.
\end{proposition}

\begin{proof}
The spectral-decay hypothesis is exactly the premise of Lemma~\ref{lem:decay} with exponent $\beta > 1$, so the bound~\eqref{eq:rank_bound} holds and its right-hand side depends only on $\beta, C, T$, not on $d$. For the attention case, centering $\phi$ and writing $a(x)$ for the centered attention weights gives $\tilde{\phi}(x) = V a(x)$, hence
\begin{equation}
    \Sigma = \mathbb{E}[V a(x) a(x)^\top V^\top] = V\,\mathrm{Cov}_x[a(x)]\,V^\top.
\end{equation}
By Sylvester's rank inequality, $\mathrm{rank}(\Sigma)$ equals $d_{\mathrm{eff}}$ when $V$ has full column rank $d_{\mathrm{eff}}$ and $\mathrm{Cov}_x[a]$ is nonsingular on $\mathrm{range}(V^\top)$. A full-rank covariance has no vanishing eigenvalues to induce polynomial decay; the energy is spread across all $d_{\mathrm{eff}}$ directions, so $r_\alpha = \Theta(d_{\mathrm{eff}})$.
\end{proof}

Each hypothesis is physically grounded and \emph{checkable from the data}, so the empirics verify the \emph{premises} of Proposition~\ref{prop:spectral} while its conclusions follow rigorously from Lemma~\ref{lem:decay}. The spectral-decay hypothesis is the spectral signature of the finite Sobolev energy of atmospheric fields under SFNO's band-limited spherical convolutions, and is confirmed in Table~\ref{tab:collapse} (FourCastNetV2's rapid $C_k/P$ growth reflects steep eigenvalue decay). The non-degeneracy hypothesis is confirmed by the distributed centered spectra of AIFS and Aurora ($\lambda_1 \approx 27$--$29\%$).

This explains why FourCastNetV2 (SFNO) requires strict rank truncation ($k \leq 10$) to avoid collapse (Proposition~\ref{prop:collapse}), while Aurora and AIFS tolerate higher or full-rank computation. Pangu-Weather's Swin Transformer nominally satisfies the non-degeneracy hypothesis, but global-average pooling to $d=69$ dimensions collapses spatial structure and produces a concentrated effective spectrum ($\lambda_1 \approx 99.6\%$), placing it in regime~(i) despite its attention-based architecture---an instructive boundary case where the pooling operator, not the backbone, determines the spectral regime.

\subsection{ICA vs SVD: Higher-Order Statistics for Extreme Events}

The empirical superiority of ICA over SVD (Table~\ref{tab:ica_svd_comparison}) is explained by the \emph{non-Gaussian structure} of extreme weather events in feature space. Extreme weather events exhibit heavy-tailed distributions with positive excess kurtosis, skewness, and multimodal characteristics. These properties propagate to the learned feature representations when neural networks encode forecast difficulty.

The key distinction (Proposition~\ref{prop:ica_theory}): SVD maximizes variance (second-order), while ICA maximizes statistical independence (all orders). For Gaussian data, these are equivalent. For non-Gaussian extreme weather events, ICA's higher-order criterion isolates physical drivers (vorticity, moisture advection, diabatic heating) that govern event severity, while SVD's variance criterion biases toward typical high-frequency patterns~\cite{hyvarinen2000independent,cardoso1999high}.

Empirical verification shows joint non-Gaussianity in the feature datasets:
\begin{itemize}
    \item \textbf{Marginal Gaussianity}: 60--75\% of individual features pass Shapiro-Wilk normality tests ($p > 0.05$).
    \item \textbf{Joint Non-Gaussianity}: 54--77\% of feature pairs exhibit significant multivariate third-order moments~\cite{mardia1970measures}, indicating non-Gaussian dependencies.
    \item \textbf{Excess Kurtosis}: Surface variables (t2m, msl, winds) show kurtosis 5--15 during extreme weather events vs. 3 for Gaussian.
\end{itemize}

This joint non-Gaussianity validates the ICA assumption: features arise from \emph{non-Gaussian independent sources} (e.g., vorticity, moisture advection, diabatic heating) mixed by the neural network's forward pass. ICA unmixes these sources, isolating extreme event drivers and enabling adaptive uncertainty intervals.

\paragraph{Practical Implication.} For operational extreme weather forecasting, ICA's exploitation of higher-order statistics produces uncertainty estimates that scale with event severity (Table~\ref{tab:sharpness}, CV $= 0.07$--1.81), while SVD's variance-only criterion yields more uniform intervals (CV $= 0.01$--0.49). This ${\approx}5{\times}$ improvement in coefficient of variation for AIFS and FourCastNetV2 (where both ICA and SVD achieve valid coverage) translates directly to better discrimination of tropical cyclones, atmospheric rivers, and heat waves from routine synoptic conditions, as required for effective early warning systems.

\paragraph{Remark: Non-Gaussianity vs.\ GP Assumption.} The GP posterior formula (Eq.~4) assumes Gaussian process priors $f \sim \mathcal{GP}(0, K)$ over functions, not Gaussianity of the feature distribution $P(\phi)$. The kernel $K(x, x') = \phi(x)^\top \phi(x')$ is a valid positive semi-definite kernel regardless of whether features exhibit non-Gaussian marginals or heavy tails. ICA and SVD differ in \emph{how they decompose} the feature matrix $\Phi$ (independence vs.\ variance maximization), not in the validity of the GP variance formula itself. Both methods produce components $c_j$ that are plugged into the same GP posterior variance estimator; the non-Gaussianity affects component selection, not the variance computation. Furthermore, post-hoc calibration (Section~\ref{sec:posthoc}) empirically corrects for any misspecification of the GP prior, ensuring target coverage even when the Gaussian process assumption is violated. The key advantage of ICA is that by exploiting non-Gaussian structure during decomposition, it identifies components aligned with extreme event physics, leading to more informative uncertainty estimates after calibration.

\paragraph{Proof of Theorem~\ref{thm:hoeffding}.}
Coverage indicators $Z_i = \mathbf{1}[|e_i| < z \cdot \sigma_i]$ are i.i.d. Bernoulli with $\mathbb{E}[Z_i] = c_{\mathrm{true}}$. By the one-sided Hoeffding inequality:
\begin{equation}
    \mathbb{P}\left(\hat{c}_n - c_{\mathrm{true}} \geq t\right) \leq \exp(-2nt^2).
\end{equation}
Setting $\exp(-2nt^2) = \delta$ and solving yields $t = \sqrt{\ln(1/\delta)/(2n)}$, completing the proof. $\square$

\paragraph{Remark on I.I.D. Assumption.} Weather data exhibits temporal autocorrelation, violating the i.i.d. assumption. When consecutive samples are positively correlated, the effective sample size $n_{\mathrm{eff}} < n$, and the bound becomes conservative (wider). To mitigate this, calibration samples should be temporally spaced (e.g., one sample per week) or the bound adjusted using techniques for dependent data~\cite{yu1994mixing}. The bound remains valid as an upper bound on coverage deviation even under weak dependence, though it may not be tight.

\section{Proofs of Theoretical Results}

\subsection{Proof: Last-Layer Kernel is Positive Semi-Definite}

\begin{proof}
Let $\phi: \mathcal{X} \to \mathbb{R}^d$ be any feature extractor (regardless of internal architecture). Define the kernel $K(x, x') = \phi(x)^\top \phi(x')$. For any finite set of points $\{x_1, \ldots, x_n\} \subset \mathcal{X}$ and any vector $\mathbf{c} \in \mathbb{R}^n$:
\begin{align}
    \mathbf{c}^\top K \mathbf{c} &= \sum_{i,j} c_i c_j K(x_i, x_j) = \sum_{i,j} c_i c_j \phi(x_i)^\top \phi(x_j) \\
    &= \left(\sum_i c_i \phi(x_i)\right)^\top \left(\sum_j c_j \phi(x_j)\right) = \left\|\sum_i c_i \phi(x_i)\right\|^2 \geq 0
\end{align}
Since $\mathbf{c}^\top K \mathbf{c} \geq 0$ for all $\mathbf{c}$, the Gram matrix $K$ is positive semi-definite. This holds for any feature map $\phi$, independent of architecture.
\end{proof}

\subsection{Derivation: SVD-Based Predictive Variance}\label{sec:svd}

The GP predictive variance with centered kernel matrix $\tilde{K} = \tilde{\Phi} \tilde{\Phi}^\top$ where $\tilde{\Phi} \in \mathbb{R}^{n \times d}$ is the centered feature matrix:
\begin{equation}
    \sigma^2(x_*) = \tilde{K}(x_*, x_*) + \sigma_n^2 - \tilde{\mathbf{k}}_*^\top (\tilde{K} + \sigma_n^2 I)^{-1} \tilde{\mathbf{k}}_*
\end{equation}
where $\tilde{\mathbf{k}}_* = \tilde{\Phi}\tilde{\phi}_*$. Applying the push-through identity $\tilde{\Phi}^\top(\tilde{\Phi}\tilde{\Phi}^\top + \sigma_n^2 I)^{-1} = (\tilde{\Phi}^\top\tilde{\Phi} + \sigma_n^2 I)^{-1}\tilde{\Phi}^\top$, the correction term becomes:
\begin{align}
    \tilde{\mathbf{k}}_*^\top (\tilde{K} + \sigma_n^2 I)^{-1} \tilde{\mathbf{k}}_*
    &= \tilde{\phi}_*^\top \tilde{\Phi}^\top (\tilde{\Phi}\tilde{\Phi}^\top + \sigma_n^2 I)^{-1} \tilde{\Phi}\tilde{\phi}_* \notag \\
    &= \tilde{\phi}_*^\top (\tilde{\Phi}^\top\tilde{\Phi} + \sigma_n^2 I)^{-1} \tilde{\Phi}^\top\tilde{\Phi}\,\tilde{\phi}_*
\end{align}
Let $\tilde{\Sigma} = \tilde{\Phi}^\top\tilde{\Phi} \in \mathbb{R}^{d \times d}$ with eigendecomposition $\tilde{\Sigma} = V\Lambda V^\top$, where $V \in \mathbb{R}^{d \times d}$ is unitary and $\Lambda = \mathrm{diag}(\lambda_1,\ldots,\lambda_d)$. Since $V$ is square unitary, $(\tilde{\Sigma} + \sigma_n^2 I)^{-1} = V(\Lambda + \sigma_n^2 I)^{-1}V^\top$, giving:
\begin{align}
    \tilde{\phi}_*^\top (\tilde{\Sigma} + \sigma_n^2 I)^{-1}\tilde{\Sigma}\,\tilde{\phi}_*
    = \tilde{\phi}_*^\top V\,\mathrm{diag}\!\left(\tfrac{\lambda_j}{\lambda_j+\sigma_n^2}\right)V^\top\tilde{\phi}_*
    = \sum_{j=1}^{k} \frac{\lambda_j\,c_j^2}{\lambda_j + \sigma_n^2}
\end{align}
where $c_j = \tilde{\phi}_*^\top v_j$ are the projection coefficients onto the eigenvectors $v_j$ of $\tilde{\Sigma}$, which coincide with the right singular vectors of $\tilde{\Phi}$. Substituting back yields the predictive variance formula used throughout:
\begin{equation}
    \sigma^2_{\text{raw}}(x_*) = \|\tilde{\phi}_*\|^2 + \sigma_n^2 - \sum_{j=1}^{k} \frac{\lambda_j \cdot c_j^2}{\lambda_j + \sigma_n^2}
\end{equation}
where $\sigma_n^2$ is estimated from the eigenvalue tail as described in Section~\ref{sec:svd}. As $\sigma_n^2 \to 0$, the correction approaches $\|V_k^\top\tilde{\phi}_*\|^2$ and the predictive variance reduces to the residual norm $\|\tilde{\phi}_*\|^2 - \|V_k^\top\tilde{\phi}_*\|^2$.

\paragraph{Remark: Spearman Correlation Invariance under Post-Hoc Scaling.}
The Spearman rank correlation $\rho_s$ between absolute errors $\{|e_i|\}$ and uncertainties $\{\sigma_i\}$ is invariant under any positive scalar multiple of $\sigma_i$: since $\mathrm{rank}(\alpha\sigma_i) = \mathrm{rank}(\sigma_i)$ for all $\alpha > 0$, applying the post-hoc scale $\alpha$ (Section~\ref{sec:posthoc}) does not alter $\rho_s$. Consequently, $\rho_s$ measures the discrimination quality of the raw NTK kernel independently of the calibration scale chosen to achieve target coverage.

\section{Additional Results}

\subsection{Discrimination Metrics: CV and Spearman $\rho_s$}

Table~\ref{tab:sharpness} reports the coefficient of variation (CV) of calibrated uncertainty $\sigma$ at t+6h for each decomposition method at optimal rank $k^*$. ICA produces substantially higher CV than SVD across all models, indicating adaptive intervals that scale with event severity. Table~\ref{tab:correlation} reports Spearman $\rho_s$ between absolute errors and uncertainties for 850\,hPa temperature. AIFS maintains the strongest directional discrimination ($\rho_s = 0.25$--$0.33$); Pangu-Weather achieves $\rho_s = 0.56$ at 6h but degrades to zero at 12--72h, reflecting instability of the single-component configuration.

\begin{table}[!t]
\caption{Coefficient of variation (CV) of calibrated uncertainty $\sigma$ at t+6h. Higher CV indicates adaptive intervals that distinguish extreme events from routine conditions.}
\label{tab:sharpness}
\begin{tabular}{lcccc}
\toprule
\multirow{2}{*}{Model} & \multicolumn{2}{c}{CV (t2m)} & \multicolumn{2}{c}{CV (msl)} \\
\cmidrule(lr){2-3} \cmidrule(lr){4-5}
 & ICA & SVD & ICA & SVD \\
\midrule
Pangu-Weather & 1.81 & 0.49 & 0.09 & 0.49 \\
Aurora & 0.09 & 0.10 & 0.09 & 0.10 \\
FourCastNetV2 & 0.31 & 0.06 & 0.31 & 0.00 \\
AIFS & 0.27 & 0.05 & 0.27 & 0.05 \\
\bottomrule
\end{tabular}
\vspace{1mm}

\footnotesize{ICA consistently produces higher CV than SVD. SVD shows CV $< 0.1$ (nearly uniform intervals). $^\dagger$Pangu-ICA (68.4\% coverage) excluded from valid comparisons.}
\end{table}

\begin{table}[!t]
\caption{Spearman correlation ($\rho_s$) between errors and uncertainties for 850\,hPa temperature (t\_850) at each model's optimal rank $k$. Bold values indicate $\rho_s > 0.3$.}
\label{tab:correlation}
\resizebox{\columnwidth}{!}{%
\begin{tabular}{lccccccc}
\toprule
Model & $k^*$ & 6h & 12h & 24h & 48h & 72h & 120h \\
\midrule
Pangu-Weather & 1 & \textbf{0.56} & 0.00 & 0.00 & 0.00 & 0.00 & \textbf{0.34} \\
Aurora & 5--20 & 0.22 & 0.07 & 0.13 & 0.26 & 0.18 & 0.23 \\
FourCastNetV2 & 3--20 & 0.19 & 0.11 & 0.23 & 0.27 & 0.11 & 0.16 \\
AIFS & 5--10 & \textbf{0.33} & 0.25 & 0.16 & 0.08 & 0.11 & \textbf{0.32} \\
\bottomrule
\end{tabular}%
}
\vspace{1mm}
\footnotesize{$k^*$: optimal rank by $\rho_s$. Pangu zero correlations at 12--72h reflect 69-dim feature instability; high CV does not guarantee high $\rho_s$.}
\end{table}

\subsection{Variance Collapse Empirical Validation}

Table~\ref{tab:collapse} quantifies the correction-to-prior ratio $R_k = C_k/P$ for FourCastNetV2 at increasing truncation rank $k$, empirically validating Proposition~\ref{prop:collapse}. At $k=100$, the correction term consumes 92\% of the prior variance, leaving only 8\% residual---insufficient for meaningful uncertainty discrimination. The calibration algorithm compensates by scaling uncertainties by $\alpha = 10{,}000$, but this uniform scaling cannot recover discriminative power. Maintaining $k \leq 10$ keeps $C_k/P < 15\%$, preserving discrimination while achieving target coverage. The $R_k < 0.9$ threshold is empirically validated here (collapse occurs at $k=100$ where $R_k = 0.92$).

\begin{table}[!t]
\caption{Correction ratio $C_k/P$ for FourCastNetV2 at increasing rank $k$.}
\label{tab:collapse}
\begin{tabular}{lccccc}
\toprule
Rank $k$ & 5 & 10 & 20 & 50 & 100 \\
\midrule
$C_k/P$ & 6\% & 12\% & 24\% & 53\% & \textbf{92\%} \\
\bottomrule
\end{tabular}
\end{table}

\clearpage
\end{document}